%% file: neurips_2026.tex
\documentclass{article}

\PassOptionsToPackage{numbers, compress}{natbib}

\usepackage[main, final]{neurips_2026}

\usepackage[utf8]{inputenc} 
\usepackage[T1]{fontenc}    
\usepackage{hyperref}       
\usepackage{url}            
\usepackage{booktabs}       
\usepackage{amsfonts}       
\usepackage{nicefrac}       
\usepackage{microtype}      
\usepackage{xcolor}         

\usepackage{bbm}
\usepackage{graphicx}
\usepackage{subcaption}
\usepackage{multicol}
\usepackage{multirow}
\usepackage{booktabs}
\usepackage{float}
\usepackage{xcolor}
\usepackage{amsmath}
\usepackage{amssymb}
\usepackage{mathtools}
\usepackage{amsthm}

\theoremstyle{plain}
\newtheorem{theorem}{Theorem}[section]
\newtheorem{proposition}[theorem]{Proposition}

\theoremstyle{definition}

\newtheorem{assumption}[theorem]{Assumption}
\theoremstyle{remark}
\newtheorem{remark}[theorem]{Remark}

\title{Dynamic Dual-Feedback Conformal Inference \\ for Time Series Forecasting}

\author{%
Songlin Du \quad Ling Luo \quad Uwe Aickelin \\
Faculty of Engineering and Information Technology \\
School of Computing and Information Systems \\
University of Melbourne \\
\texttt{sodu@student.unimelb.edu.au} \quad
\texttt{ling.luo@unimelb.edu.au} \\
\texttt{uwe.aickelin@unimelb.edu.au}
}

\begin{document}

\maketitle

\begin{abstract}
  Quantifying uncertainty in time series forecasting is particularly demanding because sequential data exhibit temporal dependence and are prone to distributional changes. Conformal inference has emerged as a powerful uncertainty quantification approach through the construction of prediction sets. Recent advances have introduced online conformal methods that adaptively adjust prediction thresholds through feedback mechanisms. However, the existing feedback mechanism typically relies solely on miscoverage indicators (actual feedback)—whether the true label falls within the interval at each time step—while overlooking the empirical prediction threshold (estimated feedback) that is derived from the oracle conformal method. In this paper, we propose \textit{Dynamic Dual-feedback Conformal Inference} (DDCI), which incorporates a dual-feedback mechanism consisting of \textit{actual feedback} and \textit{estimated feedback}. The former drives the primary adjustment of the intervals based on true observations, while the latter dampens excessive expansions or contractions by leveraging empirical thresholds from conformal inference during updates. By balancing these two signals, DDCI achieves more stable and narrower prediction intervals in sequential settings while maintaining near-nominal coverage in practice. Code is available here: \texttt{https://github.com/sdu1998/DDCI}.
\end{abstract}

\section{Introduction}
\label{sec1:intro}

Machine learning models are increasingly being applied in many real-world domains such as health care, energy and transportation, where uncertainty quantification is essential for decision-making \citep{diaz2012review, badue2021self, liu2025stochastic}. Generally speaking, prediction sets/intervals are a common approach to reflect models' uncertainty. However, traditional machine learning or statistical models themselves are difficult to quantify uncertainty, or require strong assumptions on the data distributions \citep{durbin2012time, gasthaus2019probabilistic, salinas2020deepar}. 

Conformal inference, firstly introduced by \citet{vovk1999machine}, is a non-parametric framework to construct prediction sets, which enjoys statistically marginal coverage guarantee under a lenient assumption called exchangeability. This framework offers a general methodology for converting the outputs of black-box models into prediction sets and can be applied on top of any underlying algorithm \citep{angelopoulos2023conformal}. The simplicity and generality of conformal methods have led to their wide adoption in real-world applications \citep{eklund2015application, luo2024sample}. Fundamentally, the prediction sets generated by conformal methods are guaranteed to contain the true response with a user-defined probability. However, exchangeability often fails to hold in sequential data and the validity of conformal inference may be compromised because of distributional changes arising from external factors such as seasonality or events. Motivated by this challenge, this work studies the problem of conformal inference for time series data.

In recent years, substantial effort has been devoted to developing conformal methods tailored for non-exchangeable time series data to achieve the target coverage guarantee. A prominent line of research is to update the prediction interval via feedback mechanisms. The very first method, Adaptive Conformal Inference (ACI), attains the desired coverage rate by dynamically adjusting target quantiles, treating distributional shifts as a single-parameter learning problem \citep{gibbs2021adaptive}. Following that, multiple methods have been designed to improve the predictive efficiency \citep{zaffran2022adaptive, bhatnagar2023improved, angelopoulos2023pid}. These methods update the length of the interval by determining whether the true label falls within the interval (binary feedback from current thresholds). Recently, \citet{wu2025error} proposed Error-quantified Conformal Inference (ECI), which incorporates the degree of under- or over-coverage to refine interval lengths. However, their approach does not fully exploit the magnitude of the error, thereby weakening the corrective effect of the feedback. More critically, a major limitation of all existing methods is that all of them rely exclusively on the signal from the current thresholds, which makes the update process unstable and prone to fluctuations. In other words, the interval will gradually shrink until it can no longer cover the true value at a certain point in time, after which it will expand.

Due to the limitations mentioned above, we first propose a dynamic feedback function that more precisely quantifies the \textit{actual feedback}. Secondly, to mitigate potential overreaction to this signal, we introduce an additional component, termed \textit{estimated feedback}, which is from empirical thresholds derived from split conformal methods. This second feedback introduces a counter-signal with opposite sign to the actual feedback through the empirical thresholds, serving as a control mechanism that suppresses excessive adjustments and ensures smoother, more stable updates. Our main contributions are summarised as follows:

\begin{itemize}
    \item We propose a novel online conformal methods \textit{Dynamic Dual-feedback Conformal Inference} (DDCI) for time series forecasting, which is able to quantify the actual feedback precisely and explores how empirical thresholds from conformal methods influence the stability of the updating process.
    \item Furthermore, we investigate the impact of different degrees of estimated feedback on the smoothness and efficiency of interval outputs. We also compare different kinds of estimated feedback from the original conformal method and the non-exchangeable conformal method and analyse the differences among them.
    \item We conduct considerable experiments on 3 synthetic and 5 real datasets to demonstrate that, in the vast majority of cases, our proposed methods can substantially shorten the prediction intervals without affecting the target coverage rate, compared to all existing methods.
\end{itemize}

\section{Related Work}
\label{sec2:related_work}

\subsection{Conformal Prediction}

The task of uncertainty quantification for unobserved data has been widely studied \citep{abdar2021review, smith2024uncertainty}. Conformal inference, pioneered by \citet{vovk2005algorithmic}, is a significant branch of uncertainty quantification methods. The power of conformal methods stems largely from its flexibility, simplicity and theoretical guarantee. Conformal methods are able to quantify distribution-free uncertainty under a lenient assumption of data distribution called exchangeability. Multiple variants of conformal methods and score functions have been developed to improve the adaptivity and efficiency \citep{vovk2015cross, lei2018distribution, romano2019conformalized, angelopoulos2020uncertainty, papadopoulos2002inductive, barber2021predictive, akhavan2026softmax} and have been applied to many areas \citep{vovk2021retrain, stanton2023bayesian, cresswell2024conformal, lekeufack2024conformal, zhang2024benefit, du2025duen}. A detailed introduction of conformal inference and the theory behind it are provided by~\citet{angelopoulos2023conformal, angelopoulos2024theoretical}.

\subsection{Conformal Inference under Non-exchangeability}

Conformal methods highly rely on the exchangeability, but it is often violated when they are applied to practical problems (e.g., sequential data). Recent works of developing conformal methods beyond exchangeability can be summarised into two categories. The first is to treat the past conformal scores non-equally by assigning weights, giving greater emphasis to those more similar to the present \citep{tibshirani2019conformal, xu2021conformal, barber2023conformal}. Recent approaches model temporal dependence and rely on asymptotic guarantees that require assumptions on the underlying data-generating process \cite{xu2023sequential, auer2023conformal}. The second is to adaptively update the significance level $\alpha$ or the interval directly based on the miscoverage feedback \citep{gibbs2021adaptive, zaffran2022adaptive, bhatnagar2023improved, gibbs2024conformal}. The following work \citep{angelopoulos2023pid} advanced this line of research by framing PID control theory to prospectively model conformal scores in a predictive and adaptive manner for online scenarios. \citet{angelopoulos2024online} decay the learning rate based on time steps. Following that, \citet{wu2025error} introduced an error-quantified term to update the interval by using a smoothed quantile loss. Although our work is highly inspired by the error-quantified term of \citet{wu2025error}, it differs in two aspects: (i) our feedback mechanism is residual-aware which makes it more applicable to different datasets and underlying algorithms; (ii) ours also considers an estimated feedback that is from conformal inference itself, where all existing methods did not take this empirical threshold into account. Overall, our method is therefore able to produce substantially tighter prediction sets while maintaining its validity. There are some very recent works related to this area \citep{neglia2025rescp, hu2025distribution, lee2025flow}. Closely related to our work, the concurrent work Conformal Optimistic Prediction (COP) proposes an optimistic conformal prediction framework \citep{hu2025distribution}. It leverages the distribution information of non-conformity scores as an optimistic hint within the Online Optimistic Gradient Descent framework to refine the prediction intervals.

\section{Problem Setup}
\label{sec3:problem_setup}

\textbf{Conformal Inference for Time Series.}
Under the background of time series forecasting,  we observe a sequence of covariates $x_t\in\mathcal{X}$ and responses $y_t\in\mathcal{Y}$ for $t\in \mathbb{N} = \{1, 2, 3, ..., T\}$. Then, we train an underlying model $\theta_t(\cdot)$ on the data before time $t$. As in oracle conformal methods, the conformal score function $s_t(\cdot, \cdot)$ can be designed in a flexible manner, which is able to measure the accuracy of model's prediction at time $t$. A typical choice of score function for regression problems is using the absolute error $s_t(x_t, y_t) = |y_t - \theta_t(x_t)|$. The main goal is to construct a prediction interval:

\begin{equation}
        \label{eq:interval}
        C_t(x_t) = \{ y_t\in\mathcal{Y}:s_t(x_t, y_t) \leq q_t^*\},
\end{equation}

where $q_t^*$, the empirical threshold, is the $(1-\alpha)$-th quantile from the distribution of past $w$ scores $S_t = \{s_i(x_i, y_i):t-w\leq i < t\}$. In principle, if the data is exchangeable, the interval output is expected to be marginally valid: $P(y_t \in C_t(x_t)) \geq 1-\alpha$.

Nevertheless, the challenge for time series forecasting is that the sequence data can be non-exchangeable due to external factors (e.g., distribution shifts). Therefore, a general goal is to achieve a long-run coverage guarantee that $\lim_{T \to \infty} \frac{1}{T} \sum_{t=1}^{T} err_t = \alpha,$ where $err_t = \mathbbm{1}\{y_t\notin C_t(x_t)\}$ \cite{angelopoulos2023pid, podkopaev2024adaptive}.



\textbf{Online Conformal Inference.}
The online update of prediction interval started from ACI \citep{gibbs2021adaptive} and evolved into the quantile tracking component inside the  conformal PID control method \citep{angelopoulos2023pid}, called online gradient descent (OGD). It considers an optimisation problem to minimise the quantile loss $\ell_t(q_t)$,

\begin{equation}
    q_{t+1} := q_t - \eta_t \nabla \ell_t(q_t) = q_t + \eta_t \bigl(\text{err}_t - \alpha \bigr),
    \label{formula: OGD}
\end{equation}

where $\eta_t$ is a non-negative learning rate and $\ell_t(q_t) = (s_t - q_t)\bigl(\mathbf{1}\{s_t > q_t\} - \alpha \bigr)$ means the $(1-\alpha)$-th quantile loss. 
Simply put, if the interval $C_t(x_t)$ fails to cover the true label at time $t$, the interval is expanded to make it conservative in the next step; otherwise, it is shortened.

\section{Methodology}

\subsection{Dynamic Dual-feedback Conformal Inference}

In the standard OGD update rule (Eq.~\ref{formula: OGD}), the threshold adjustment relies solely on the binary coverage feedback ($\text{err}_t - \alpha$). While effective, this signal is too coarse as it only indicates whether the prediction interval covers the true label, without reflecting the magnitude or reliability of the error. Additionally, relying on a single feedback can make the update process unstable and inclined to volatility.

To address this limitation, we augment the update rule with two complementary forms of feedback. To clarify, $q_t$ and $q_t^*$ denote the current threshold and the estimated threshold from conformal inference respectively. The first is the \emph{actual feedback}, expressed as 
\begin{equation}
    \mathcal{A}(e_t) = \rho_t h(c \cdot e_t),\ \ \ \ \ \rho_t = \min(\frac{|e_t|}{\max_{t-w \leq i,j \leq t} \, \lvert s_i - s_j \rvert}, 1) \in [0, 1],
    \label{formula:actual}
\end{equation}
which incorporates the observed error $e_t = s_t-q_t$ and uses a squashing function $h(\cdot)$ to weight the update by the severity of miscoverage. This squashing function is capable of mapping the error to $[-d, d]$ and is origin-symmetric, e.g., $tanh(x)$. This ensures that large deviations from the threshold lead to stronger corrective actions. The parameter $c$ governs the steepness of the curve.

Furthermore, we employ the second feedback, termed \emph{estimated feedback}, given by
\begin{equation}
    \mathcal{E}(e_t^*) = \left(1 - \rho_t^*\right) h(c \cdot e_t^*),\ \ \ \ \ \rho_t^*=\frac{|e_t^*|}{\max_{t-w \leq i,j \leq t} \, \lvert s_i - s_j \rvert}\in [0, 1],
    \label{formula:estimated}
\end{equation}
which provides a stabilising signal based on the estimated error $e_t^*=s_t-q_t^*$ and $q_t^*$ is the empirical $(1-\alpha)$-th quantile of scores $S_t = \{s_i(x_i, y_i):t-w\leq i < t\}$. 

For theoretical simplicity, we set ${\max_{t-w \leq i,j \leq t} \, \lvert s_i - s_j \rvert} = 2B$ based on the Assumption \ref{assumption:1}. In practice, we use the maximum and minimum from the past window. This empirical quantile $q_t^*$ acts as a stabilising anchor. By referencing the robust empirical threshold $q_t^*$, it provides a local distributional context that prevents the interval from overfitting to instantaneous noise. Notice the term $(1 - \rho_t^*)$ acts as a natural adaptive gate. When $e_t^*$ is small ($|e_t^*| \ll 2B$), the term allows the estimated feedback to effectively dampen noise and stabilise the interval. When $e_t^*$ are large ($|e_t^*| \to 2B$), the term vanishes, automatically diminishing the estimated feedback. This allows the algorithm to prioritise the actual feedback for rapid adaptation without interference. Thus, the mechanism intelligently balances stability and responsiveness based on the magnitude of the prediction error relative to the global bound $2B$. A visual comparison of these two feedback is given in Fig.~\ref{fig:feedback}, with $tanh(x)$ as an example squashing function.

\begin{figure*}[tb]
    \centering
    \begin{subfigure}{0.4\textwidth}
        \centering
        \includegraphics[width=\linewidth]{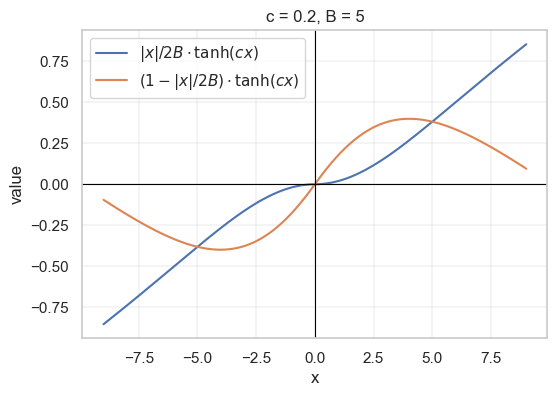}
        \vspace{-5mm}
        \caption{$B=5$}
    \end{subfigure}
    \begin{subfigure}{0.4\textwidth}
        \centering
        \includegraphics[width=\linewidth]{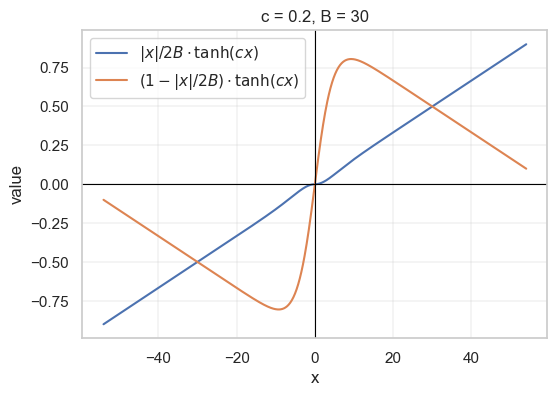}
        \vspace{-5mm}
        \caption{$B=30$}
    \end{subfigure}
    \vspace{-2mm}
    \caption{Functional shape comparison of the actual feedback $\tfrac{|x|}{2B}\tanh(cx)$ (blue) and the estimated feedback $(1-\frac{|x|}{2B})\tanh(cx)$ (orange) with $c=0.2$ and $B = 5 $ or $30$. The two functions are plotted on the same axis to highlight their distinct regimes of dominance: the former grows in magnitude and dominates for large |x|, whereas the latter exhibits attenuation and dominates near the origin.}
    \label{fig:feedback}
\end{figure*}

Therefore, we propose \textit{Dynamic Dual-feedback Conformal Inference} (DDCI), based on the OGD method (Eq.~\ref{formula: OGD}) and the dual-feedback mechanism, the update rule is given below:
\begin{equation}
    \begin{aligned}
    q_{t+1} := q_t + \eta_t(\text{err}_t - \alpha)
+ \eta_t ( 
    \underbrace{\mathcal{A}(s_t - q_t)}_{\text{Actual feedback}} -
    \underbrace{\epsilon\mathcal{E}(s_t - q_t^*)}_{\text{Estimated feedback}}
),
    \end{aligned}
    \label{formula:DDCI}
\end{equation}
where $\epsilon > 0$ is a parameter to regulate the influence of the estimated feedback term. The estimated feedback enters the update rule with an opposite sign to the actual feedback term, forming a push-pull mechanism that prevents over-reaction. Specifically, when both $\mathcal{A}(e_t)$ and $\mathcal{E}(e^*_t)$ are positive, the actual feedback drives $q_t$ upward while the estimated feedback applies a counteracting downward pull. This mechanism is particularly effective when $|q_t - q^*_t|$ is small: $\mathcal{A}$ and $\mathcal{E}$ have similar magnitudes, allowing the estimated feedback to dampen excessive expansion. Conversely, when $q_t$ deviates substantially from $q^*_t$ (indicating a genuine distribution shift), the actual feedback dominates and drives the threshold to adapt quickly.
As shown in Fig.~\ref{fig:feedback}, the actual feedback term $\mathcal{A}(e_t)$ is highly sensitive to sudden shifts or anomalous values of $s_t$, which can exert a strong influence on the update. By balancing these signals, the update adapts flexibly to distribution changes while maintaining robustness, thereby constructing prediction intervals that are both responsive and stable. 

To further understand this dynamic feedback mechanism, a mathematical decomposition of the additional feedback mechanism within Appendix~\ref{appendix:regret} (Stabilised Control term) reveals that the mechanism acts as an adaptive gating function: it damps high-frequency noise when $|e_t|$ is small while restoring corrective force during significant distribution shifts.

\subsection{Theoretical Analysis}
\label{mt:TA}

The main theoretical analysis of regret bound and long-run coverage of the proposed method DDCI is presented in this section. More details and the analysis on the parameter $\epsilon$ are provided in the Appendix \ref{app:ta}. 

At first, we start from a general assumption for conformal scores, which sets a boundary \citep{angelopoulos2023pid, wu2025error}.

\begin{assumption}
    \textit{For any } $t \in \mathbb{N}_+$\textit{, there exists } $0<B<\infty$ \textit{ such that } $s_t \in [-B, B]$.
    \label{assumption:1}
\end{assumption}

\begin{assumption}
    \textit{The origin-symmetric squashing function $h(\cdot)$ is chosen to be a bounded ($|h(\cdot)|\leq d$ for $x \in R$), and $L_h$-Lipschitz continuous function. Consequently, for any $c > 0$ and $x, y \in \mathbb{R}$: $|h(cx) - h(cy)| \leq L_h |cx - cy| = cL_h|x-y|$.}
    \label{assumption:2}
\end{assumption}

Based on these assumptions, there is a boundary for $q_t$, which are given in the following Proposition.

\begin{proposition}
    \textit{Fix an initial threshold $q_1 \in [-B, B]$. Assume that $\eta_t \leq \frac{2B}{1-\alpha+d(1+\epsilon)}$. The update rule in Eq.~\ref{formula:DDCI} with arbitrary non-negative learning rate $\eta_t$ satisfies that,}
    \vspace{-1mm}
    \begin{equation}
    -B-(\alpha + d(1+\epsilon)) M_{t-1} \le q_t \le B + (1-\alpha + d(1+\epsilon)) M_{t-1} \quad \forall t \ge 1,
    \end{equation}

    \text{where } $M_{0}$ = 0, $M_{t} = \max_{1 \leq r \leq t} \eta_{r} 
    \quad \text{for } t \geq 1.$ The upper bound $\eta_t \leq \frac{2B}{1-\alpha+d(1+\epsilon)}$ ensures that a single update step cannot move $q_t$ across the entire range $2B$.
    \label{proposition:1} 
    
\end{proposition}


\subsubsection{Regret Bound}

At first, we analyse the theoretical stability of DDCI using the framework of Dynamic Regret \citep{daniely2015strongly, orabona2019modern, zhao2024adaptivity}. Dynamic regret generalises the standard regret framework to non-stationary environments by bench-marking the learner against a sequence of time-varying optimal comparators $u_t$, rather than a single fixed parameter \citep{zinkevich2003online}.


\begin{theorem}
\label{theorem:regret}
    DDCI achieves the following dynamic regret bound for arbitrary $u_t$ with a fixed learning rate $\eta$ and a small $\epsilon$ ($\epsilon \leq 1$):
    \vspace{-1mm}
    \begin{equation}
    \begin{aligned}
        \frac{1}{T}\sum_{t=1}^T (\ell_t(q_t) - \ell_t(u_t)) \le& \frac{2B^2}{\eta T} + \frac{2B+\eta K_+}{\eta T}\sum_{t=1}^T ||u_{t+1} - u_t||  + 2\epsilon^2 \gamma^2\frac{\eta}{T}\sum_{t=1}^T ||q_t - q_t^*||^2 \\
    & + \frac{(2B+\eta K_+) \epsilon \gamma}{T}\sum_{t=1}^T ||q_t - q_t^*||+ (2B+\eta K_+)d + \eta(1+2d^2),
    \end{aligned}
    \end{equation}
\end{theorem}

where $\gamma = cL_h+\frac{d}{2B}$ and $K_+=1-\alpha+d(1+\epsilon)$. The term $\sum_{t=1}^T \|u_{t+1} - u_t\|$ represents the path length of the optimal comparator sequence and captures the non-stationarity of the environment. The proof is provided in the Appendix~\ref{appendix:regret}.

A critical distinction between DDCI and single feedback methods (e.g., ECI) lies in their sensitivity to the stochastic nature of conformal scores. Typically, when the underlying algorithm is well-trained, the conformal score $s_t$ tends to be highly stochastic, exhibiting significant variance. Existing methods often suffer from updates that are explicitly dependent on the magnitude of observed error $||s_t - q_t||$. Their structure couples the algorithmic performance with the inherent uncertainty of the data. Consequently, even if the algorithm converges to the optimal parameter, high-variance noise in $s_t$ inevitably limits the algorithm's stability and makes it sensitive to outliers. In contrast, our analysis reveals that the stability of DDCI can be controlled by the empirical quantile $q_t^*$. This formulation mitigates the irreducible data noise from the learning process but its effectiveness highly relies on the quality of $q_t^*$. Theorem~\ref{theorem:regret} presents a worst-case adversarial bound containing a non-vanishing penalty $\mathcal{O}(d)$, which is the limitation of introducing the additional feedback. Paying this $\mathcal{O}(d)$ in the worst-case bound is the cost for a strong insurance policy against extreme scenarios. However, this cost would be small since it is linearly related to the steepness parameter $c$.

\begin{remark}
     Specifically, during time intervals where the score distribution remains invariant (i.e., locally exchangeable, which means that the distribution of the current score $s_t$ is consistent with the scores in its preceding window of size $w$ in our case), the empirical $q_t^*$ acts as a theoretically consistent surrogate for the true local measure, and our analysis of the additional feedback mechanism implies that using $q_t^*$ can weaken the influence from the $s_t$. This capability to adapt to local stationarity is formally substantiated by our Localised Static Regret analysis in the Appendix \ref{appendix:LSReg}. Furthermore, under locally exchangeable periods, because $q_t^*$ is an aggregate order statistic computed over a trailing window of size $w$, its asymptotic variance is approximated by: $\text{Var}(q_t^*) \approx \frac{\alpha(1-\alpha)}{w \cdot [f_t(q_t^*)]^2}$ where $f_t(q_t^*)$ is the probability density. This formulation demonstrates that the variance of $q_t^*$ shrinks on the order of $O(1/w)$, making it more stable than using $s_t$ merely (which has an $O(1)$ variance). By considering $q_t^*$, the feedback mechanism can reduce the influence from the irreducible data noise (e.g., outliers), thereby promoting greater stability during the learning process. As $q_t$ converges to $q_t^*$, the difference term in our bound vanishes asymptotically.
     \label{remark:LSReg}
\end{remark}

\subsubsection{Long-run Coverage Analysis}

Given the proposition above, the following theorem shows the worst-case upper bound of \textit{long-run average miscoverage rate}.

\begin{theorem}[]
    \textit{Let $\{\eta_t\}_{t \geq 1}$ be an arbitrary positive sequence. Under Assumption \ref{assumption:1}, \ref{assumption:2} and $\eta_t \leq \frac{2B}{1-\alpha+d(1+\epsilon)}$, the prediction interval generated by Eq.~\ref{formula:DDCI} with learning rate $\eta_t$ satisfies:}
    \vspace{-2mm}
    \begin{equation}
    \begin{aligned}
    |\frac{1}{T}\sum_{t=1}^{T}& (\text{err}_t - \alpha)|\leq \frac{(2B + (1+2d(1+\epsilon))M_{T}) \, \lVert \Delta_{1:T} \rVert_{1}}{T} + \underbrace{(c L_h +  \epsilon \gamma)(2B+K_+M_{T})}_{Bias},
    \end{aligned}
    \end{equation}
    where 
    $\| \Delta_{1:T} \|_1 = |\eta_1^{-1}| + \sum_{t=2}^{T} |\eta_t^{-1} - \eta_{t-1}^{-1}|$, $M_T = \max_{1 \leq r \leq T} \eta_r$, $K_+=1-\alpha+d(1+\epsilon)$, $\gamma = cL_h+\frac{d}{2B}$, $c$ and $d$ are the steepness parameter and the range of the squashing function $h(\cdot)$.
    \label{theorem:coverage}
\end{theorem}

The first term follows the analysis from \citet{angelopoulos2024online} and this term vanishes asymptotically as $T \to \infty$ as long as the decay of the learning rate is not too fast (e.g., sublinear decay such as $\eta_t \propto t^{-1/2}$). The rest part reveals that the empirical miscoverage is bounded by a constant bias, which is caused by introducing additional feedback terms. A detailed proof is given in the Appendix~\ref{appendix:coverage}. While a non-zero asymptotic bias appears as a theoretical limitation, it is a deliberate and controllable bias-variance trade-off. Specifically, this bias can be decoupled into two distinct components governed by our dual-feedback mechanism, as detailed in the following remark.

\begin{remark}

    The bias in Theorem~\ref{theorem:coverage} is a worst-case quantity derived under adversarial conditions, obtained by bounding all error terms by the domain diameter. In practice, this bound is far from tight. The bias is from two sources and governed by the hyperparameters $c$ and $\epsilon$: (1) $cL_h(2B+K_+M_{T})$: The actual feedback introduces an empirical sum $\frac{1}{T}\sum_{t=1}^{T} |h(c \cdot e_t)|$, which is bounded via the Lipschitz continuity of $h(\cdot)$. This worst-case bound scales linearly with $c$, so selecting a small $c$ directly suppresses this bias term without affecting the qualitative behaviour of the feedback mechanism, as corroborated by the sensitivity analysis in the Appendix~\ref{results:steepness}.
    (2) $\epsilon\gamma\frac{1}{T}\sum_{t=1}^{T} ||q_t - q_t^*||$ (Structural Lag Bias): This term stems from the estimated quantile. During non-stationary distribution shifts, the empirical threshold $q_t^*$ naturally lags behind the instantaneous optimal quantile. However, under locally exchangeable regimes, $q_t^*$ stabilises towards the true optimal quantile. In such periods, this lag penalty collapses into a statistical variance penalty bounded by $\mathcal{O}({\epsilon}/{\sqrt(w)})$.
\end{remark}

\section{Experiments}
\label{sec:experiments}
\subsection{Settings}
\label{sec:exp_settings}

\textbf{Datasets.} Three synthetic datasets (described in the Section~\ref{sec:synthetic})  and five public datasets are used to evaluate methods (Google/Amazon/Microsoft stock price, Delhi temperature and Electricity, described in the Appendix~\ref{app:dataset_description}). Due to the limit of pages, the results and analysis for Delhi temperature and Electricity datasets are given in the Appendix~\ref{app:addition_results}.


\textbf{Underlying Time Series Forecasting Algorithms.} Four algorithms are used: Prophet \citep{taylor2018forecasting}, AutoRegressive (AR) \citep{box2015time}, Theta \citep{assimakopoulos2000theta} and Transformer \citep{vaswani2017attention}. The chosen models span different paradigms of time series forecasting.

\textbf{Online Conformal Baselines.} (1) ACI \citep{gibbs2021adaptive}: updates the intervals via adjusting the quantile threshold to select from calibration set; (2) OGD family \citep{bhatnagar2023improved}: uses an iterative optimization algorithm that updates model parameters incrementally; (3) Conformal PID Control \citep{angelopoulos2023pid}: uses control theory and is capable to model scores to adjust intervals; (4) ECI \citep{wu2025error}:  uses a smoothed quantile loss to update intervals; (5) COP \citep{hu2025distribution}: uses the distribution information to optimise the update.

\textbf{Common Setup.} We set the target miscoverage rate to 10\% and construct asymmetric prediction intervals \citep{romano2019conformalized}. For baselines, the parameter settings (\textit{T\_burnin}, window size, learning rate) are consistent with those described in the open-source implementations \citep{angelopoulos2023pid, wu2025error}. The results and analysis of varying significance levels and learning rates are in Appendix~\ref{results:siglevels} and~\ref{results:learningrates}.

\textbf{DDCI Setup.} The squashing function employed here is $tanh(\frac{1}{2}x)$ and the hyperparamter $\epsilon$ is set to 0.2. For DDCI, the estimated threshold $q_t^*$ is from the original split conformal method which is under the risk of under-coverage in time series settings and the size of the calibration set is equal to \textit{T\_burnin} which is used as a warm-up data in baselines. For DDCI-Nex, the $q_t^*$ follows the design of $NexCP$ \citep{barber2023conformal}, where the weights for conformal scores are $w_i = \frac{0.99^{\,t-i}}{\sum_{j=t-w}^{t-1} 0.99^{\,t-j}}, \quad t-w \leq i \leq t-1$. We also test another squashing function $softsign(\cdot)$ and the results are given in the Appendix~\ref{results:siglevels}. The results and analysis about varying window length $w$ and the squashing function range $d$ and the steepness parameter $c$ are given in the Appendix~\ref{results:window_length} and~\ref{results:function_range} and~\ref{results:steepness}.

\renewcommand{\arraystretch}{1.2}
\begin{table*}[htbp!]
\centering
\caption{Results on the three synthetic datasets generated under the Prophet model with nominal level 10\%. All experiments are repeated 5 times, and the table reports the mean and standard deviation. The best result (shortest width) is marked in bold, and the second-best result is marked with an underline.}
\label{table:synthetic1}
\resizebox{\textwidth}{!}{
\begin{tabular}{cccccccccc}
\toprule
\multirow{2}{*}{\textbf{Method}} 
& \multicolumn{3}{c}{\textbf{Synthetic Data 1}}
& \multicolumn{3}{c}{\textbf{Synthetic Data 2}}
& \multicolumn{3}{c}{\textbf{Synthetic Data 3}}
\\ \cmidrule(lr){2-10}
& \begin{tabular}[c]{@{}c@{}}Coverage\\ (\%)\end{tabular}
& \begin{tabular}[c]{@{}c@{}}Avg.\\ width\end{tabular}
& \begin{tabular}[c]{@{}c@{}}Winkler\\ Score\end{tabular}
& \begin{tabular}[c]{@{}c@{}}Coverage\\ (\%)\end{tabular}
& \begin{tabular}[c]{@{}c@{}}Avg.\\ width\end{tabular}
& \begin{tabular}[c]{@{}c@{}}Winkler\\ Score\end{tabular}
& \begin{tabular}[c]{@{}c@{}}Coverage\\ (\%)\end{tabular}
& \begin{tabular}[c]{@{}c@{}}Avg.\\ width\end{tabular}
& \begin{tabular}[c]{@{}c@{}}Winkler\\ Score\end{tabular}
\\ \midrule

OGD 
& $90.0 \pm 0.12$ & $11.01 \pm 0.50$ & $12.82 \pm 0.54$
& $89.9 \pm 0.17$ & $17.28 \pm 7.59$ & $28.13 \pm 18.54$
& $85.0 \pm 4.78$ & $371.43 \pm 350.57$ & $506.07 \pm 541.8$
\\

SF-OGD
& $90.7 \pm 0.16$ & $75.75 \pm 0.11$ & $84.93 \pm 0.14$
& $90.3 \pm 0.23$ & $76.83 \pm 1.50$ & $86.57 \pm 2.44$
& $89.4 \pm 0.98$ & $183.13 \pm 138.25$ & $213.06 \pm 162.6$
\\

decay-OGD
& $89.7 \pm 0.82$ & $21.34 \pm 2.10$ & $23.32 \pm 2.44$
& $84.3 \pm 4.90$ & $44.05 \pm 34.24$ & $56.04 \pm 47.25$
& $59.6 \pm 22.53$ & $419.60 \pm 307.13$ & $2202.6 \pm 2148.2$
\\

PID
& $89.9 \pm 0.11$ & $13.25 \pm 1.18$ & $15.19 \pm 1.33$
& $89.8 \pm 0.08$ & $18.04 \pm 6.53$ & $25.74 \pm 11.62$
& $89.8 \pm 0.13$ & $124.17 \pm 100.70$ & $144.22 \pm 115.73$
\\

ECI
& $90.0 \pm 0.10$ & $13.07 \pm 1.39$ & $15.07 \pm 1.53$
& $90.0 \pm 0.12$ & $19.95 \pm 7.15$ & $30.58 \pm 11.92$
& $81.9 \pm 17.84$ & $88.85 \pm 54.96$ & $119.52 \pm 61.22$
\\



COP
& $89.8 \pm 0.02$ & $10.07 \pm 1.14$ & $11.86 \pm 1.33$
& $90.0 \pm 0.16$ & $16.69 \pm 7.17$ & $26.53 \pm 11.52$
& $89.7 \pm 0.15$ & $91.25 \pm 51.53$ & $109.43 \pm 60.39$
\\ 
\midrule

\textbf{DDCI}
& $89.9 \pm 0.02$
& \textbf{8.08 $\pm$ 0.82}
& \textbf{9.82 $\pm$ 1.04}
& $90.0 \pm 0.19$
& \textbf{13.15 $\pm$ 5.35}
& \textbf{22.93 $\pm$ 9.78}
& $89.8 \pm 0.10$
& \textbf{73.33 $\pm$ 41.58}
& \textbf{90.00 $\pm$ 49.68}
\\

\textbf{DDCI-Nex}
& $89.9 \pm 0.03$
& \underline{$8.35 \pm 0.84$}
& \underline{$10.07 \pm 1.00$}
& $89.9 \pm 0.14$
& \underline{$15.36 \pm 6.75$}
& \underline{$25.20 \pm 11.18$}
& $89.8 \pm 0.10$
& \underline{$87.10 \pm 51.51$}
& \underline{$104.78 \pm 59.85$}
\\ \bottomrule

\end{tabular}}
\end{table*}

\textbf{Evaluation Metrics.} \emph{Validity}: Average coverage rate on the test set; \emph{Efficiency}: Average width and the Winkler score \citep{winkler1972decision} of prediction intervals on the test set; \emph{Stability}: The mean absolute successive difference (MASD) of interval width. Let $W_t$ denote the interval width at time $t$. MASD is defined as $\operatorname{MASD}_W = \frac{1}{n-1}\sum_{t=2}^{n} \left|W_t-W_{t-1}\right|.$ A smaller $\operatorname{MASD}_W$ indicates less step-to-step oscillation and hence smoother interval updates.

\subsection{Experimental Results}

\subsubsection{Synthetic Datasets}
\label{sec:synthetic}

We constructed three synthetic datasets exhibiting different forms of non-stationarity: (1) \textbf{Random Walk Trend:} a smoothed stochastic trend with time-varying volatility, generated as $X_t = X_{t-1} + \sigma_t \eta_t,\eta_t \sim \mathcal{N}(0,1),$ where the volatility evolves according to $\sigma_t = \sigma_0\bigl(1 + \kappa\,|\sin(2\pi t / P)|\bigr).$ The time-varying term $\sigma_t$ introduces slow periodic volatility modulation, capturing smoother drift and cyclical heteroskedasticity; (2) \textbf{Exponential Trend:} a multiplicative-growth process \(X_t = X_{t-1}\exp(\mu + \sigma \varepsilon_t),\ \varepsilon_t \sim \mathcal{N}(0,1)\); (3) \textbf{Changepoint:} a piecewise log-linear process with two regime shifts. For changepoints \(c_1 = 600\) and \(c_2 = 1200\), $\log X_t = \log X_{t-1} + \mu_{(t)} + \sigma_{(t)}\varepsilon_t, \varepsilon_t \sim \mathcal{N}(0,1),$ where \((\mu_{(t)},\sigma_{(t)})\) vary across segments, producing moderate structural breaks.

The results are averaged over 5 runs under Prophet model and given in the Table~\ref{table:synthetic1}. Across all synthetic settings, DDCI consistently delivers the most efficient and stable intervals while maintaining the target coverage. In Setting 1 and 2, DDCI achieves the narrowest intervals and lowest Winkler scores, with DDCI-Nex providing the second-best performance. Under the more challenging changepoint setting (Setting 3), where all methods experience increased uncertainty, DDCI still produces the most informative intervals, substantially outperforming the baselines. Notably, ECI (actual feedback only) loses the validity. Overall, the results show that DDCI is robust across different forms of non-stationarity and yields clear efficiency gains over existing methods.

\subsubsection{Public Datasets}

The results of baselines can also be found within \citet{wu2025error} and \citet{hu2025distribution}. Overall, all methods achieve coverage rates close to the target level of 90\%.

Table~\ref{table:stock} reports the performance across three stock datasets with a miscoverage rate of $\alpha=10\%$. The baseline comparisons reveal distinct limitations in existing methods. While ACI satisfies the theoretical coverage requirement, it results in infinite interval widths. Similarly, traditional online methods, such as the OGD family and PID, tend to be overly conservative, yielding unnecessarily wide intervals. Although the ECI generally produces narrower intervals than these earlier approaches, its performance exhibits instability. In contrast, the proposed DDCI and DDCI-Nex consistently deliver the most competitive results across all datasets and underlying forecasting algorithms. 
The robustness of DDCI is further highlighted by its ability to maintain tight intervals even under complex architectures like Prophet and Transformer, drastically outperforming the best baselines.

\begin{table*}[h!]
\centering
\renewcommand{\arraystretch}{1.2}
\caption{The experimental results on three datasets (Google, Amazon, Microsoft).}
\label{table:stock}
\resizebox{\textwidth}{!}{%
\begin{tabular}{cccccccccccccc}
\toprule
\multirow{2}{*}{\textbf{Dataset}} & \multirow{2}{*}{\textbf{Method}} 
& \multicolumn{3}{c}{\textbf{Prophet Model}}
& \multicolumn{3}{c}{\textbf{AR Model}}
& \multicolumn{3}{c}{\textbf{Theta Model}}
& \multicolumn{3}{c}{\textbf{Transformer}} \\ \cline{3-14} 
& & \begin{tabular}[c]{@{}c@{}}Coverage\\ (\%)\end{tabular} 
& \begin{tabular}[c]{@{}c@{}}Avg. \\ width\end{tabular} 
& \begin{tabular}[c]{@{}c@{}}Winkler\\ Score\end{tabular} 
& \begin{tabular}[c]{@{}c@{}}Coverage\\ (\%)\end{tabular} 
& \begin{tabular}[c]{@{}c@{}}Avg.\\ width\end{tabular} 
& \begin{tabular}[c]{@{}c@{}}Winkler\\ Score\end{tabular} 
& \begin{tabular}[c]{@{}c@{}}Coverage\\ (\%)\end{tabular} 
& \begin{tabular}[c]{@{}c@{}}Avg.\\ width\end{tabular} 
& \begin{tabular}[c]{@{}c@{}}Winkler\\ Score\end{tabular} 
& \begin{tabular}[c]{@{}c@{}}Coverage\\ (\%)\end{tabular} 
& \begin{tabular}[c]{@{}c@{}}Avg.\\ width\end{tabular} 
& \begin{tabular}[c]{@{}c@{}}Winkler\\ Score\end{tabular} \\ 
\midrule

\multirow{9}{*}{\textbf{Google}}
& ACI
& 90.0 & $\infty$ & $\infty$
& 89.8 & $\infty$ & $\infty$
& 90.5 & $\infty$ & $\infty$
& 90.2 & $\infty$ & $\infty$ \\

& OGD
& 89.7 & 57.60 & 74.74
& 90.7 & 33.76 & 70.96
& 89.9 & 31.49 & 47.88
& 90.1 & 109.27 & 148.76 \\

& SF-OGD
& 89.6 & 58.92 & 77.52
& 89.9 & 28.31 & \textbf{41.93}
& 90.0 & 34.04 & 51.01
& 90.1 & 88.30 & 114.34 \\

& decay-OGD
& 89.9 & 77.23 & 92.36
& 90.2 & 46.54 & \underline{58.27}
& 90.2 & 55.33 & 67.89
& 89.9 & 120.25 & 144.38 \\

& PID
& 89.9 & 56.13 & 71.18
& 89.8 & 64.19 & 83.58
& 89.3 & 75.71 & 97.37
& 89.8 & 77.36 & 106.78 \\

& ECI
& 89.9 & 56.07 & 75.46
& 89.7 & 19.96 & 69.11
& 89.6 & 30.93 & 47.64
& 89.9 & 70.93 & 110.29 \\

& COP
& 89.7 & 49.72 & 63.75
& 89.6 & \textbf{19.87} & 69.10
& 89.3 & 30.25 & 47.09
& 89.7 & 61.93 & 99.47 \\
\cmidrule{2-14}

& \textbf{DDCI}
& 89.8 & \textbf{39.20} & \textbf{52.80}
& 89.6 & \underline{19.92} & 68.97
& 89.4 & \underline{29.24} & \textbf{46.38}
& 90.0 & \underline{51.18} & \underline{86.68} \\

& \textbf{DDCI-Nex}
& 89.8 & \underline{39.69} & \underline{53.30}
& 89.6 & 19.93 & 68.97
& 89.3 & \textbf{29.15} & \underline{46.57}
& 89.9 & \textbf{51.03} & \textbf{86.51} \\

\midrule

\multirow{9}{*}{\textbf{Amazon}}
& ACI
& 90.2 & $\infty$ & $\infty$
& 89.8 & $\infty$ & $\infty$
& 89.7 & $\infty$ & $\infty$
& 90.1 & $\infty$ & $\infty$ \\

& OGD
& 89.6 & 55.15 & 67.57
& 89.9 & 19.10 & \textbf{30.08}
& 89.8 & 18.07 & 28.23
& 89.4 & 52.68 & 69.27 \\

& SF-OGD
& 89.5 & 61.47 & 73.94
& 89.9 & 24.44 & 34.31
& 90.0 & 23.88 & 34.46
& 89.3 & 56.56 & 72.48 \\

& decay-OGD
& 89.9 & 97.22 & 107.79
& 89.7 & 20.23 & \underline{30.74}
& 89.2 & 17.49 & 28.56
& 89.8 & 93.16 & 109.19 \\

& PID
& 89.9 & 45.90 & 56.02
& 89.7 & 49.40 & 62.72
& 88.6 & 71.11 & 87.88
& 89.8 & 56.05 & 71.10 \\

& ECI
& 90.1 & 47.00 & 59.74
& 89.5 & 17.12 & 34.23
& 89.7 & 17.46 & \underline{27.86}
& 89.9 & 49.07 & 68.77 \\

& COP
& 89.6 & 39.86 & 49.40
& 89.5 & 17.09 & 34.14
& 89.6 & 17.21 & \textbf{27.66}
& 89.5 & 42.73 & 61.32 \\
\cmidrule{2-14}

& \textbf{DDCI}
& 89.7 & \textbf{30.87} & \underline{41.06}
& 89.3 & \textbf{17.04} & 34.30
& 89.5 & \textbf{17.05} & 27.90
& 89.7 & \textbf{34.46} & \textbf{51.61} \\

& \textbf{DDCI-Nex}
& 89.7 & \underline{30.92} & \textbf{40.95}
& 89.3 & \underline{17.06} & 34.29
& 89.6 & \underline{17.06} & 27.87
& 89.7 & \underline{35.16} & \underline{52.25} \\

\midrule

\multirow{9}{*}{\textbf{Microsoft}}
& ACI
& 90.0 & $\infty$ & $\infty$
& 90.7 & $\infty$ & $\infty$
& 89.9 & $\infty$ & $\infty$
& 89.9 & $\infty$ & $\infty$ \\

& OGD
& 89.9 & 3.77 & 5.24
& 90.7 & 4.37 & 7.54
& 89.9 & 2.49 & \textbf{3.78}
& 89.9 & 3.80 & 6.98 \\

& SF-OGD
& 90.2 & 6.91 & 8.07
& 90.0 & 12.62 & 14.48
& 90.2 & 6.98 & 8.23
& 89.9 & 7.08 & 8.61 \\

& decay-OGD
& 90.1 & 6.02 & 7.30
& 90.2 & 7.97 & 8.94
& 90.2 & 4.92 & 6.00
& 89.7 & 5.41 & 7.05 \\

& PID
& 90.0 & 6.30 & 7.66
& 89.6 & 5.60 & \textbf{6.96}
& 89.8 & 4.89 & 6.41
& 89.9 & 6.16 & 8.13 \\

& ECI
& 90.4 & 4.85 & 6.14
& 90.2 & 3.76 & 8.50
& 90.2 & 2.48 & 3.94
& 89.9 & 4.35 & 7.45 \\

& COP
& 89.7 & 3.78 & 5.16
& 89.7 & \textbf{2.66} & \underline{7.38}
& 89.4 & \textbf{2.38} & \underline{3.81}
& 89.6 & 3.77 & 6.89 \\
\cmidrule{2-14}

& \textbf{DDCI}
& 89.8 & \textbf{3.26} & \textbf{4.55}
& 89.8 & 2.91 & 7.72
& 89.8 & \underline{2.39} & 3.82
& 89.9 & \textbf{3.34} & \textbf{6.43} \\

& \textbf{DDCI-Nex}
& 89.8 & \underline{3.40} & \underline{4.71}
& 89.8 & \underline{2.90} & 7.67
& 89.8 & 2.39 & 3.84
& 89.8 & \underline{3.47} & \underline{6.51} \\

\bottomrule
\end{tabular}}
\end{table*}

Fig.~\ref{fig:GOOGL} presents a visualisation of comparison result between ECI (single feedback only) and DDCI (ours) on the Google stock dataset using the Prophet model. Generally, ECI tends to produce wider ranges, which often appears overly conservative, with noticeable fluctuations in the interval boundaries. By contrast, DDCI yields narrower and more stable intervals, reducing redundancy while maintaining strong alignment with the true series.

\begin{figure*}[b]
    \centering
    \includegraphics[width=0.7\linewidth]{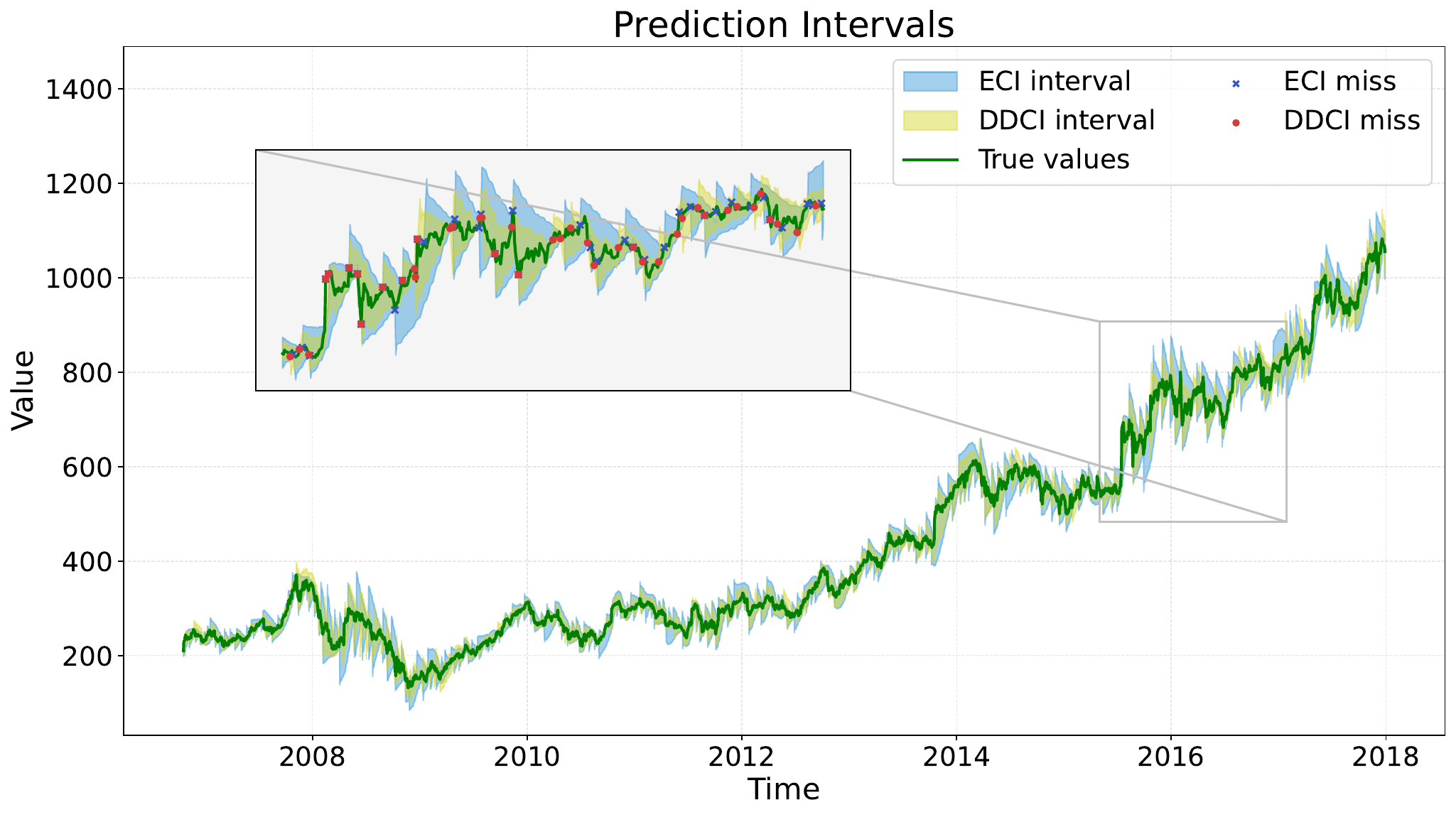}
    \caption{Visualisation of comparison results between ECI (single feedback only) and DDCI (proposed, $\epsilon=0.2$) on Google stock dataset under Prophet model. The blue crosses and red points represent the miscovered points of each method.}
    \label{fig:GOOGL}
\end{figure*}

In conclusion, the proposed methods generally achieve tighter prediction intervals while maintaining validity. However, their improvements are less pronounced under the AR model. These observations can be attributed to the behaviour of the conformal estimated threshold $q_t^{*}$. When the conformal prediction yields relatively conservative intervals, characterised by achieving the target coverage with a large average width, the estimated threshold tends to overestimate local uncertainty. 

\subsubsection{Stability Evaluation and Ablation Study}
\label{sec:stability}

We show \emph{stability evaluation} on all five public datasets under Prophet in Tab.~\ref{tab:stability}. We compare PID and ECI with three configurations of DDCI. DDCI with $c=0$ disables both additional feedback terms because $h(0)=0$, and is therefore equivalent to OGD under the same proportional learning rate. DDCI with $\epsilon=0$ retains only the actual feedback $\mathcal{A}(e_t)$, whereas the full DDCI combines the actual and estimated feedbacks.

The full DDCI achieves the lowest MASD on all five datasets. Relative to the actual-feedback-only configuration, adding the estimated feedback reduces MASD by $2.3\%$--$33.5\%$. The actual feedback alone is not uniformly smoother than the $c=0$ configuration, whereas combining the actual and estimated feedbacks produces the lowest successive interval variation in every setting. Across these ablations, adding the estimated feedback changes coverage by at most $0.1$ percentage points and mean width by at most $5.5\%$. Together with the coverage and efficiency results in Tables~\ref{table:stock} and~\ref{res_merged_temp_elec}, this shows that the stability improvement is not obtained through a material loss of validity or wider intervals. All results under all models are provided in Tab.~\ref{tab:additional_stability} (Appendix~\ref{visual:epsilon}), where DDCI achieves the lowest MASD in 19 of the 20 dataset--forecaster settings.

\begin{table*}[]
\centering
\small
\caption{Stability evaluation (MASD, $\downarrow$) and ablation study on the five public datasets under Prophet.}
\label{tab:stability}
\begin{tabular}{lccccc}
\toprule
Method & GOOGL & AMZN & MSFT & Delhi & Electricity \\
\midrule
PID
& 8.55 & 6.94 & 1.02 & 1.05 & 0.0645 \\

ECI
& 9.66 & 7.78 & 0.81 & 0.43 & 0.0628 \\

DDCI ($c=0$; no additional feedback)
& \underline{8.32}
& \underline{6.64}
& \underline{0.63}
& 0.32
& 0.0622 \\

DDCI ($\epsilon=0$; actual feedback only)
& 9.93
& 8.05
& 0.68
& \underline{0.31}
& \underline{0.0617} \\

DDCI ($\epsilon=0.2$; dual feedback)
& \textbf{6.89}
& \textbf{5.35}
& \textbf{0.54}
& \textbf{0.28}
& \textbf{0.0603} \\
\bottomrule
\end{tabular}
\end{table*}

\subsubsection{Sensitivity Analysis}
We examine the role of the parameter $\epsilon$ in the update rule Eq.~\ref{formula:DDCI}, taking Google stock data as an illustrative case. Fig.~\ref{fig:ablation_size} shows that increasing the parameter $\epsilon$ consistently enlarges the average interval size across all four models, confirming its role in controlling the balance between efficiency and smoothness. 
Noticeably, the influence of $\epsilon$ becomes remarkable when it is greater than 0.5. 
The observed robustness to small $\epsilon$ aligns perfectly with our theoretical regret bound. 
Consequently, for small values, the theoretical cost on the regret bound is minimal due to this quadratic scaling. 
This provides a strong justification for our parameter choice: keeping $\epsilon$ small allows the model to benefit from the stability mechanism while incurring a negligible penalty on the dynamic regret.

DDCI-Nex yields less sensitivity to $\epsilon$ by employing NexCP for threshold estimation. This shows the influence of estimated threshold $q_t^*$ on the updating process. This highlights that a more reliable estimated threshold is crucial for enabling the dual-feedback mechanism to function effectively. The apparent sensitivity shown in Fig.~\ref{fig:ablation_size} is not primarily driven by~$\epsilon$ itself. The different trends across panels arise because each base model (AR, Prophet, Theta, Transformer) employs a distinct learning rate (0.05, 0.5, 0.1, 0.5, respectively). A visualisation of how $\epsilon$ can influence the predictive efficiency are displayed in the Appendix~\ref{visual:epsilon}. In this example, we set $\epsilon$ to 0 and 0.8 respectively, which can be regarded as a visualisation for ablation study. Fig.~\ref{fig:ablation_coverage} displays that the effect of the parameter $\epsilon$ is relatively minor compared to its influence on interval size. 
This indicates that varying $\epsilon$ mainly adjusts efficiency (interval width) without substantially undermining validity.

\begin{figure*}[]
    \centering
    \includegraphics[width=0.95\linewidth]{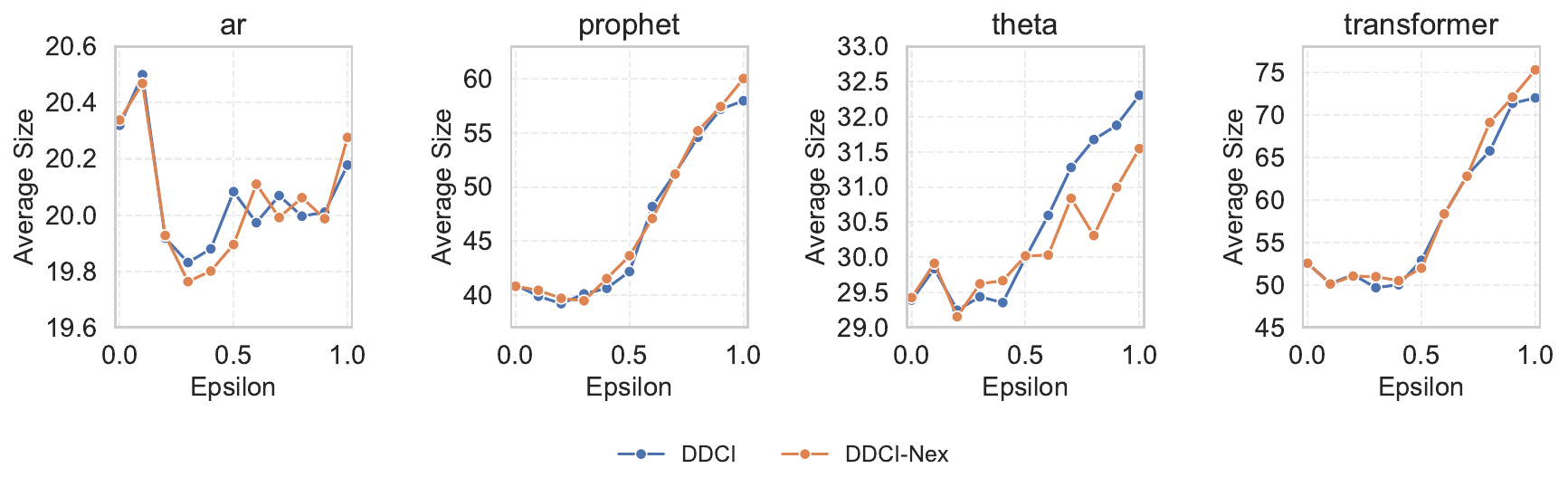}
    \caption{The influence of different smoothing parameter $\epsilon$ on the size of prediction interval.}
    \label{fig:ablation_size}
\end{figure*}
\begin{figure*}[]
    \centering
    \includegraphics[width=0.95\linewidth]{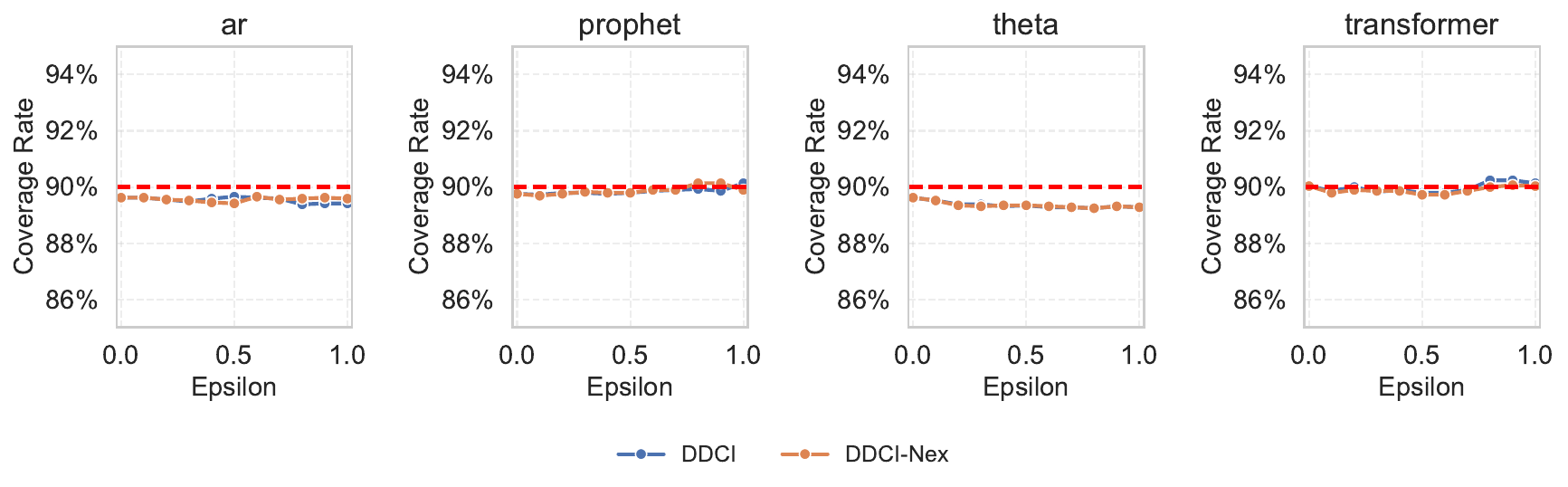}
    \caption{The influence of different smoothing parameter $\epsilon$ on the coverage rate.}
    \label{fig:ablation_coverage}
\end{figure*}

\section{Conclusion}
This work introduces a novel framework that incorporates a dual-feedback mechanism. The mechanism combines \textit{actual feedback} and \textit{estimated feedback}, whose signs are intentionally set in opposition to counterbalance increases and decreases when updating the intervals. While the use of \textit{estimated feedback} may slightly widen the intervals when the true label lies inside them, it effectively prevents unnecessary expansions, leading to more stable interval adjustments. Further work could improve the DDCI method by adaptively selecting the parameter $\epsilon$ in response to distributional changes, thereby enhancing both stability and efficiency.

\bibliography{ref}
\bibliographystyle{abbrvnat}

\newpage

\appendix

\section{Theoretical Analysis}
\label{app:ta}

Here, we provide the proofs for theoretical analysis in Section~\ref{mt:TA}. The proofs for  Proposition are given in the \ref{appendix:propositions}. The proofs for regret bound analysis are in the \ref{appendix:regret} and \ref{appendix:LSReg}. The proof for long-run coverage analysis is provided in the \ref{appendix:coverage}. The analysis for the additional parameter $\epsilon$ is in the \ref{appendix:epsilon}.

Recall the update rule:

\begin{equation*}
    \begin{aligned}
    q_{t+1} := q_t + \eta_t(\text{err}_t - \alpha) 
+ \eta_t ( 
    \underbrace{\rho_th(c \cdot e_t)}_{\text{Actual feedback}} -
    \underbrace{\epsilon\left(1 - \rho_t^*\right) h(c \cdot e_t^*)}_{\text{Estimated feedback}}
), \ \ \ \ \rho_t, \rho_t^*\in [0, 1]
    \end{aligned}
    \label{formula:DDCI_appendix}
\end{equation*}
where $\epsilon > 0$ is a parameter to regulate the influence of the estimated feedback term. The actual and estimated feedbacks operate in opposing directions to establish a push-pull mechanism, ensuring dynamic equilibrium. In other words, the actual feedback acts as a driving force to correct immediately, while the estimated feedback serves as a stabilising force to prevent over-adjustment. For the theoretical analysis, we replace $\max_{t-w \leq i,j \leq t} \, \lvert s_i - s_j \rvert$ in $\rho_t$ and $\rho_t^*$ with its upper bound 2B according to the Assumption~\ref{assumption:1}. This yields a conservative bound; in practice $\max_{t-w \leq i,j \leq t} \, \lvert s_i - s_j \rvert$ may be substantially smaller, which explains the empirical efficiency of DDCI.

\subsection{Propositions \ref{proposition:1}}
\label{appendix:propositions}

\textit{Proof of Proposition~\ref{proposition:1}.} We prove this by induction, following the structure of similar arguments in prior work \citep{angelopoulos2024online, wu2025error}. First, $q_1 \in [-B, B]$ by assumption. Assume that $\eta_t \leq \frac{2B}{1-\alpha+d(1+\epsilon)}$ and $\epsilon \ge 0$. Next fix any $t \ge 1$. Meanwhile, the empirical quantile satisfies $|q_t^*|\leq B$. Then we analyse the update in four cases based on the relative positions of $s_t$, $q_t$, and $q^*_t$.

\medskip
(1): $s_t \ge q_t$ (Interval expands, error = $1 - \alpha$)
\begin{itemize}
    \item $q_t^* \ge s_t \ge q_t$:
    \begin{align*}
        q_{t+1} &= q_t + \eta_t(1 - \alpha) + \eta_t\left(\rho_th(c(s_t - q_t)) - \epsilon\left(1 - \rho_t^*\right)h(c(s_t - q_t^*))\right) \\
        &\le q_t + \eta_t(1 - \alpha) + d\eta_t(1 + \epsilon) \\
        &\le B + (1 - \alpha + d(1 + \epsilon))M_t
    \end{align*}

    \item $s_t \ge q_t \ge q_t^*$ or $s_t \ge q_t^* \ge q_t$:
    \begin{align*}
        q_{t+1} &= q_t + \eta_t(1 - \alpha) + \eta_t\left(\rho_th(c(s_t - q_t)) - \epsilon\left(1 - \rho_t^*\right)h(c(s_t - q_t^*))\right) \\
        &\le q_t + \eta_t(1 - \alpha) + d(1+\epsilon)\eta_t \\
        &\le B + (1 - \alpha +d(1+\epsilon))M_t
    \end{align*}
\end{itemize}

\medskip
(2): $s_t \le q_t$ (Interval shrinks, error = $-\alpha$)
\begin{itemize}
    \item $q_t^* \le s_t \le q_t$:
    \begin{align*}
        q_{t+1} &= q_t + \eta_t(-\alpha) + \eta_t\left(\rho_th(c(s_t - q_t)) - \epsilon\left(1 - \rho_t^*\right)h(c(s_t - q_t^*))\right) \\
        &\ge q_t + \eta_t(-\alpha) - d\eta_t(1 + \epsilon) \\
        &\ge -B - (\alpha + d(1+\epsilon))M_t
    \end{align*}

    \item $s_t \le q_t \le q_t^*$ or $s_t \le q_t^* \le q_t$:
    \begin{align*}
        q_{t+1} &= q_t + \eta_t(-\alpha) + \eta_t\left(\rho_th(c(s_t - q_t)) - \epsilon\left(1 - \rho_t^*\right)h(c(s_t - q_t^*))\right) \\
        &\ge q_t + \eta_t(-\alpha) - d\eta_t(1+\epsilon) \\
        &\ge -B - (\alpha + d(1+\epsilon))M_t
    \end{align*}
\end{itemize}

Then combining all cases gives:

\begin{equation*}
    -B - K_-M_{t-1} \le q_t \le B + K_+M_{t-1} \quad \forall t \ge 1,
\end{equation*}

\text{where } $M_{0}$ = 0, $M_{t} = \max_{1 \leq r \leq t} \eta_{r} \quad \text{for } t \geq 1, K_-=\alpha + d(1+\epsilon) \ \text{and}\ K_+=1 - \alpha +d(1+\epsilon).$

\paragraph{Remark on the inductive bound.} The bound 
$|q_t| \leq B + K_+ M_{t-1}$ (due to $\alpha$ is usually small, we use $K_+$ as the upper bound of the absolute value) leverages the structural property that the maximum learning rate $M_{t-1}$ acts as a buffer: once $q_t$ approaches $\pm B$, the self-correcting nature of the dual-feedback signals (when $q_t$ exceeds $B$, the actual feedback $\mathcal{A(\cdot)}$ becomes negative; symmetric for $q_t < -B$) prevents unbounded drift. The constraint $\eta_t~\leq~\frac{2B}{K_+}$  ensures that no single update can move $q_t$ across the entire range $2B$.

\subsection{Dynamic Regret for Theorem~\ref{theorem:regret}}
\label{appendix:regret}

To rigorously analyse the dynamic regret of DDCI, we cast the algorithm within the framework of Online Mirror Descent (OMD) \citep{orabona2019modern}. Let $Br_{\psi}(\cdot,\cdot)$ be the Bregman divergence with respect to $\psi(x): x\in \mathcal{R}$ and assume $\psi$ to be closed and $\lambda$-strongly convex. Specifically, we define: $g_t = g_t^{OGD} + g_t^{AF} - \epsilon g_t^{EF}$. Here $g_t^{OGD} \in \nabla \ell_t(q_t)$ is a subgradient of the pinball loss. The algorithm thus performs inexact subgradient descent: instead of using $g_t^{OGD}$ directly, it uses the perturbed direction $g_t = g_t^{OGD} + g_t^{AF} - \epsilon g_t^{EF}$, where $g_t^{AF} - \epsilon g_t^{EF}$ is the bias. We use Euclidean distance  $Br_\psi(q, q_t) = \frac{1}{2}\|q - q_t\|^2$, which is the Bregman divergence induced by $\psi(q)=\frac{1}{2}||q||^2$.

\textit{Proof of Theorem~\ref{theorem:regret}.} 
{
From the one-step inequality for OMD with competitor $u_t$ and fixed learning rate $\eta$, we have
\begin{align*}
    \eta(\ell_t(q_t) - \ell_t(u_t))
    &\le \eta g_t^{OGD}(q_t-u_t) = \eta (g_t - (g_t^{AF}-\epsilon g_t^{EF}))(q_t-u_t) \\
    &= \eta g_t(q_t-u_t) - \eta (g_t^{AF}- \epsilon g_t^{EF})(q_t-u_t) \\
\end{align*}
Applying the standard OMD analysis to the first term, we bound it using the Bregman divergence $Br_\psi$:
\begin{equation*}
\eta g_t(q_t - u_t) \le Br_\psi(u_t, q_t) - Br_\psi(u_t, q_{t+1}) + \frac{\eta^2}{2}||g_t||^2
\end{equation*}

For the second term, by Cauchy-Schwarz and the bounded domain Assumption~\ref{assumption:1} ($||q_t - u_t|| \le 2B+K_+M_{t-1}$), we have:
\begin{equation*}
- \eta (g_t^{AF}- \epsilon g_t^{EF})(q_t - u_t) \le \eta ||g_t^{AF}- \epsilon g_t^{EF}|| \cdot ||q_t - u_t|| \le \eta (2B+K_+M_{t-1}) ||g_t^{AF}- \epsilon g_t^{EF}||
\end{equation*}

Combining these, we obtain:
\begin{align*}
\eta(\ell_t(q_t) - \ell_t(u_t)) \le & Br_\psi(u_t, q_t) - Br_\psi(u_{t+1}, q_{t+1}) + Br_\psi(u_{t+1}, q_{t+1}) - Br_\psi(u_t, q_{t+1}) \\
&+ \frac{\eta^2}{2}||g_t||^2 + \eta (2B+\eta K_+) ||g_t^{AF}- \epsilon g_t^{EF}||
\end{align*}

Dividing by $\eta$ and averaging over time $T$, we have:
\begin{align*}
\frac{1}{T}\sum_{t=1}^T (\ell_t(q_t) - \ell_t(u_t)) &\le \frac{2B^2}{\eta T} + \frac{1}{\eta T}\sum_{t=1}^T \langle \nabla\psi(q_{t+1}) - \nabla\psi(q_1), u_t - u_{t+1} \rangle \\ & + \frac{1}{T}\sum_{t=1}^T\frac{\eta}{2} ||g_t||^2 + \frac{2B+\eta K_+}{T}\sum_{t=1}^T ||g_t^{AF} - \epsilon g_t^{EF}|| \\
&\le \frac{2B^2}{\eta T} + \frac{2B+ \eta K_+}{\eta T}\sum_{t=1}^T ||u_{t+1} - u_t|| \\
&+ \frac{1}{T}\sum_{t=1}^T\frac{\eta}{2} ||g_t||^2 + \frac{2B+\eta K_+}{T}\sum_{t=1}^T ||g_t^{AF} - \epsilon g_t^{EF}||
\end{align*}

Now we analyse the term $g_t$:

(1) \textbf{The Gradient Descent ($g_t^{OGD}$):} The standard descent direction derived from the pinball loss, representing the fundamental learning signal: $||g_t^{OGD}|| = ||\nabla \ell_t(q_t)|| \leq 1;$

(2) \textbf{Stabilised Control ($g_t^{AF}-\epsilon g_t^{EF}$):} A correction term governed by the dual feedback mechanism. It dynamically ``re-calibrates" the update magnitude by referencing the estimated threshold, ensuring stability in stationary regimes and responsiveness during drifts.
\begin{align*}
    g_t^{AF}-\epsilon g_t^{EF}
    &= \frac{|e_t|}{2B}h(c \cdot e_t) - \epsilon (1-\frac{|e_t^*|}{2B})h(c \cdot e_t^*) \\
    &=\left[ (1+\epsilon)\frac{|e_t|}{2B} - \epsilon \right] h(c \cdot e_t) - \epsilon \left( \frac{|e_t| - |e_t^*|}{2B} \right) h(c \cdot e_t) + \epsilon \left(1 - \frac{|e_t^*|}{2B}\right) \Big[ h(c \cdot e_t) - h(c \cdot e_t^*) \Big]\\
\end{align*}

The identity decomposition above serves not only to facilitate mathematical bounding but also to elucidate how the dual-feedback mechanism achieves an adaptive toggle between smoothness and responsiveness at a fundamental level.

Specifically, the first term in the decomposition acts as the primary control term for the system update. A closer inspection of its coefficient, $\left[ (1+\epsilon)\frac{|e_t|}{2B} - \epsilon \right]$, reveals that $\epsilon \leq 1$ and it essentially functions as an adaptive gating switch:

\begin{itemize}
    \item Locally exchangeable Regimes: When the instantaneous prediction error $|e_t|$ is small (i.e., $|e_t| < \frac{\epsilon}{1+\epsilon} 2B$), this coefficient becomes negative. This implies that the overall update direction actively opposes the native actual feedback $h(c \cdot e_t)$, allowing the estimated feedback (anchored by $q_t^*$) to dominate. This mechanism forcefully suppresses overreactions to single-step data noise, providing a dynamic dampening effect that ensures interval stability;
    \item Drift Regimes: When encountering a severe distribution shift that significantly inflates the instantaneous error $|e_t|$, the coefficient turns positive, and the dampening effect $\epsilon(1 - \frac{|e_t|}{2B})$ decays rapidly. Consequently, the actual feedback regains dominance, driving the threshold $q_t$ to rapidly expand or contract to capture the new underlying data distribution.
\end{itemize}

Meanwhile, the remaining two terms in the decomposition perfectly isolate the structural bias introduced by the empirical quantile. By leveraging the triangle inequality and the Lipschitz continuity of $h(\cdot)$, these terms can be cleanly bounded proportionally to $||q_t - q_t^*||$. This naturally lays the groundwork for the $\gamma$-dependent worst-case bound derived below.

Then combining boundedness and Lipschitz continuity, we have the worst-case bound $||g_t^{AF} - \epsilon g_t^{EF}|| \le d + \epsilon \gamma ||q_t - q_t^*||$, where $\gamma = c L_h + \frac{d}{2B}$. The constant $d$ here comes from $\left[ (1+\epsilon)\frac{|e_t|}{2B} - \epsilon \right] h(c \cdot e_t)$ when $e_t$ always reaches the maximum in the worst-case. However, this term is linearly correlated to the steepness parameter $c$ (Lipschitz continuity), which can be a small value without influence the efficiency based on the sensitivity analysis on $c$ (Appendix~\ref{results:steepness}).


Recalling the inequality $||g_t||^2 = ||g_t^{OGD} + (g_t^{AF} - \epsilon g_t^{EF})||^2 \le 2||g_t^{OGD}||^2 + 2||g_t^{AF} - \epsilon g_t^{EF}||^2$, we substitute $||g_t^{AF} - \epsilon g_t^{EF}||$ and $||g_t||^2$ back into the regret bound derived in the previous step. Then we obtain the final dynamic regret:

\begin{align*}
    \frac{1}{T}\sum_{t=1}^T (\ell_t(q_t) - \ell_t(u_t)) \le& \frac{2B^2}{\eta T} + \frac{2B+ \eta K_+}{\eta T}\sum_{t=1}^T ||u_{t+1} - u_t|| + 2\epsilon^2 \gamma^2\frac{\eta}{T}\sum_{t=1}^T ||q_t - q_t^*||^2 \\
    & + \frac{(2B+\eta K_+) \epsilon \gamma}{T}\sum_{t=1}^T ||q_t - q_t^*||+ (2B+\eta K_+)d + \eta(1+2d^2).
\end{align*}

The term $\sum_{t=1}^T \|u_{t+1} - u_t\|$ represents the path length of the optimal comparator sequence. This term captures the non-stationarity of the environment.
}

\subsection{Localised Static Regret for Remark~\ref{remark:LSReg}}
\label{appendix:LSReg}
Then, we start showing the Locally Static Regret (LSReg) analysis of DDCI. During the time interval $\tau$ where the score distribution remains invariant (i.e., locally exchangeable), the $q_t^*$ from conformal prediction reduces to a fixed optimal solution $u^*$. This supports the stabilising role of $q_t^*$ as an empirical anchor.

\textbf{Corollary A.1.} \textit{Let $\tau$ be a time interval of length $k$ and assume the score distribution remains invariant within this interval. For a fixed optimal decision $u^*$ within $\tau$, based on the inexact gradient analysis in Theorem \ref{theorem:regret}, DDCI achieves the following local regret bound:}
\begin{align*}
    \frac{1}{k} \sum_{t=\tau}^{\tau+k-1} (\ell_t(q_t) - \ell_t(u^*)) \le& \frac{(2B+\eta K_+)^2}{2\eta k}+ \frac{\eta}{k}\sum_{t=\tau}^{\tau+k-1} \Big( 2\epsilon^2 \gamma^2 ||q_t - q_t^*||^2 \Big)  + (2B+\eta K_+)d  \\
    & + \frac{(2B+\eta K_+) \epsilon \gamma}{k}\sum_{t=\tau}^{\tau+k-1} ||q_t - q_t^*||+ \eta(1+2d^2).
\end{align*}

\textit{Proof of Corollary A.1.}
For a fixed $u^*$ within the time interval $\tau$, we follow the inexact gradient decomposition $g_t = g_t^{OGD} + g_t^{AF} - \epsilon g_t^{EF}$ as established in the proof of Theorem \ref{theorem:regret}. 

Starting from the one-step inequality and summing over $t = \tau$ to $\tau+k-1$, we have:
\begin{align*}
    \frac{1}{k} \sum_{t=\tau}^{\tau+k-1} (\ell_t(q_t) - \ell_t(u^*)) \le & \frac{1}{\eta k} \sum_{t=\tau}^{\tau+k-1} \frac{||q_t - u^*||^2 - ||q_{t+1} - u^*||^2}{2} \\
    &+ \frac{1}{k} \sum_{t=\tau}^{\tau+k-1} \frac{\eta}{2} ||g_t||^2 + \frac{1}{k} \sum_{t=\tau}^{\tau+k-1} (2B+\eta K_+) ||g_t^{AF}-\epsilon g_t^{EF}||
\end{align*}

Because the competitor $u^*$ is fixed, the first term forms a perfect telescoping sum. Using the bounded domain assumption $||q_\tau - u^*|| \le 2B+K_+M_{t-1}$, we bound the telescoping term as:
\begin{equation*}
    \frac{||q_\tau - u^*||^2 - ||q_{\tau+k} - u^*||^2}{2\eta k} \le \frac{||q_\tau - u^*||^2}{2\eta k} \le \frac{(2B+\eta K_+)^2}{2\eta k}
\end{equation*}

Notice that the dynamic path length term $\sum ||u_{t+1} - u_t||$ vanishes completely since $u_{t+1} = u_t = u^*$. 

Finally, substituting the bound $||g_t^{AF}-\epsilon g_t^{EF}|| \le d + \epsilon \gamma ||q_t - q_t^*||$ into the inequality, we can get the desired local regret bound:
\begin{align*}
    \frac{1}{k} \sum_{t=\tau}^{\tau+k-1} (\ell_t(q_t) - \ell_t(u^*)) \le& \frac{(2B+\eta K_+)^2}{2\eta k}+ \frac{\eta}{k}\sum_{t=\tau}^{\tau+k-1} \Big( 2\epsilon^2 \gamma^2 ||q_t - q_t^*||^2 \Big) + \eta(1+2d^2) \\
    & + \frac{(2B+\eta K_+) \epsilon \gamma}{k}\sum_{t=\tau}^{\tau+k-1} ||q_t - q_t^*|| + (2B+\eta K_+)d.
\end{align*}
\hfill

\subsection{Coverage Analysis for Theorem~\ref{theorem:coverage}}
\label{appendix:coverage}

    
\textit{Proof of Theorem~\ref{theorem:coverage}. Define $\Delta_r :=\eta_r^{-1}-\eta_{r-1}^{-1}$ for $r \geq 1$, with the convention $\eta_0^{-1}:=0$. Then we can get the long-run coverage bound:} 
{
\footnotesize
\begin{align*}
\left| \frac{1}{T} \sum_{t=1}^T (\text{err}_t - \alpha) \right|
&= \left| \frac{1}{T} \sum_{t=1}^T \left( \sum_{r=1}^t \Delta_r \right) \cdot \eta_t (\text{err}_t - \alpha) \right| \\
&= \left| \frac{1}{T} \sum_{r=1}^T \Delta_r \left( \sum_{t=r}^T \eta_t (\text{err}_t - \alpha) \right) \right| \\
&= \left| \frac{1}{T} \sum_{r=1}^T \Delta_r \left( \sum_{t=r}^T \left[
q_{t+1} - q_t - \eta_t \rho_t h(c(s_t - q_t)) + \eta_t\epsilon (1-\rho_t^*) h(c(s_t - q_t^*)))
\right] \right) \right| \\
&\leq \left| \frac{1}{T} \sum_{r=1}^T \Delta_r (q_{T+1} - q_r) \right|
+ \left| \frac{1}{T} \sum_{r=1}^T \Delta_r \sum_{t=r}^T \eta_t (\rho_t h(c(s_t - q_t))
- \epsilon (1-\rho_t^*) h(c(s_t - q_t^*))) \right| \\
&\leq \left| \frac{1}{T} \sum_{r=1}^T \Delta_r (q_{T+1} - q_r) \right|
+ \left| \frac{1}{T} \sum_{t=1}^T (\rho_t h(c(s_t - q_t))
- \epsilon (1-\rho_t^*) h(c(s_t - q_t^*)) \right| \\
&\leq \frac{(2B + (1+2d(1+\epsilon))M_{T}) \, \lVert \Delta_{1:T} \rVert_{1}}{T} + \frac{1}{T}\sum_{t=1}^{T} |h(c\cdot e_t)| + \epsilon \gamma \frac{1}{T}\sum_{t=1}^{T} ||q_t-q_t^*|| \\
&\leq \frac{(2B + (1+2d(1+\epsilon))M_{T}) \, \lVert \Delta_{1:T} \rVert_{1}}{T} + (c L_h + \epsilon \gamma) (2B+K_+M_{T}),
\end{align*}

where 
    $\| \Delta_{1:T} \|_1 = |\eta_1^{-1}| + \sum_{t=2}^{T} |\eta_t^{-1} - \eta_{t-1}^{-1}|$, $K_+=1-\alpha+d(1+\epsilon)$, $\gamma = c L_h+\frac{d}{2B}$ and $c$ is a steepness parameter inside the squashing function, $M_T = \max_{1 \leq r \leq T} \eta_r$.
}

For two bias terms $|h(c\cdot e_t)|$ and $||q_t-q_t^*||$, $2B+K_+M_{T}$ is the worst-case bound according to the Proposition~\ref{proposition:1}.

\subsection{The Heuristic Analysis of $\epsilon$}
\label{appendix:epsilon}


To ensure the update direction is consistent with the primary learning signal, we analyse the condition where the magnitude of the actual feedback exceeds that of the estimated feedback. Incorporating the weighting coefficients from Eq.~\ref{formula:DDCI}, the condition is:$$\frac{|e_t|}{2B} |h(c e_t)| > \epsilon \left(1 - \frac{|e_t^*|}{2B}\right) |h(c e_t^*)|.$$ Leveraging the relationship $s_t - q_t = (s_t - q_t^*) - q_t + q_t^*$, this simplifies to a dynamic bound on $\epsilon$:$$\epsilon < \frac{|e_t^*+q_t^*-q_t|}{2B - |e_t^*|}\frac{|h(c(e_t^*+q_t^*-q_t))|}{|h(ce_t^*)|}.$$

To make this bound tractable, we consider the algorithm in the stationary regime where the parameter $q_t$ fluctuates within a small convergence radius $\xi$ around the estimated threshold, i.e., $|q_t - q_t^*| \le \xi$. We analyse the worst-case scenario where the parameter deviation weakens the learning signal. 

Using the first-order Taylor expansion for the squashing function $h(c \cdot x)$ around $e_t^*$, we approximate the numerator term:
$$ h(c(e_t^* + q_t^* - q_t)) \approx h(c e_t^*) - c \cdot h'(c e_t^*) \cdot \xi. $$
The condition can be simplified to:

$$ \epsilon < \frac{|e_t^*-\xi|}{2B - |e_t^*|} \left( \frac{h(c e_t^*) - c \xi h'(c e_t^*)}{h(c e_t^*)} \right) = (1 - \underbrace{\frac{c \xi h'(c e_t^*)}{h(c e_t^*)}}_{\text{Safety Margin}})\frac{|e_t^*-\xi|}{2B - |e_t^*|} $$

This derived bound reveals two critical constraints on $\epsilon$. First, the term involving $\xi$ acts as a safety margin. This margin accounts for inherent parameter oscillation near equilibrium. Second, the bound dynamically tightens as the prediction error $|e_t^*|$ diminishes. Specifically, the scaling factor $\frac{|e_t^*-\xi|}{2B-|e_t^*|}$ decreases as the error magnitude diminishes. Consequently, the feasible range for $\epsilon$ shrinks significantly in regimes when $e_t^*$ is small. To guarantee stability across all scenarios, it is theoretically preferable to select a small $\epsilon$. This justifies our empirical choice (e.g., $\epsilon=0.2$) to ensure the actual feedback remains dominant, even in the presence of convergence noise.

\section{Additional Experimental results}

\subsection{Description of Datasets}
\label{app:dataset_description}

\begin{itemize}
    \item \textbf{Amazon and Google}: contain thirty blue-chip stock prices, with daily opening prices forecasted on a log scale, spanning the period from January 1, 2006 to December 31, 2014. \url{https://www.kaggle.com/datasets/camnugent/sandp500}.
    \item \textbf{Microsoft}: includes daily stock opening prices from April 1, 2015 to May 31, 2021. \url{https://www.kaggle.com/datasets/camnugent/sandp500}.
    \item \textbf{Delhi temperature}: provides daily measurements of temperature (averaged over eight readings taken at three-hour intervals), humidity, wind speed, and atmospheric pressure in Delhi, from January 1, 2003 to April 24, 2017, collected via the Weather Underground API. \url{https://www.kaggle.com/datasets/sumanthvrao/daily-climate-time-series-data}.
    \item \textbf{Electricity demand} \citep{harries1999splice}: records electricity demand in New South Wales at half-hour intervals from May 7, 1996 to December 5, 1998.
\end{itemize}

\subsection{Additional Experimental Results}

\subsubsection{Delhi temperature and Electricity Demand}
\label{app:addition_results}

On the Delhi temperature dataset, DDCI variants consistently achieve the narrowest prediction intervals (except AR model) while preserving nominal coverage. The improvements over the strongest baselines are smaller in magnitude compared with stock datasets but remain consistent, demonstrating the robustness and adaptability of DDCI in settings with relatively stable time series such as temperature.

For the electricity demand dataset, the performance of the DDCI variants is broadly comparable to that of the ECI baselines rather than superior. Across all forecasting models, DDCI and DDCI-Nex achieve interval widths that are similar to the best ECI methods while maintaining the target coverage level. Although the improvements observed in the stock datasets are not replicated here, the results indicate that DDCI remains competitive without loss of validity. This suggests that the performance of DDCI in the electricity demand setting remains stable and consistent with strong baseline approaches.

\begin{table*}[h!]
\centering
\renewcommand{\arraystretch}{1.2}
\caption{The experimental results on Delhi Temperature and Electricity Demand datasets.}
\label{res_merged_temp_elec}
\resizebox{\textwidth}{!}{%
\begin{tabular}{cccccccccccccc}
\toprule
\multirow{2}{*}{\textbf{Dataset}} & \multirow{2}{*}{\textbf{Method}} 
& \multicolumn{3}{c}{\textbf{Prophet Model}}
& \multicolumn{3}{c}{\textbf{AR Model}}
& \multicolumn{3}{c}{\textbf{Theta Model}}
& \multicolumn{3}{c}{\textbf{Transformer}} \\ \cline{3-14} 
& & \begin{tabular}[c]{@{}c@{}}Coverage\\ (\%)\end{tabular} 
& \begin{tabular}[c]{@{}c@{}}Avg. \\ width\end{tabular} 
& \begin{tabular}[c]{@{}c@{}}Winkler\\ Score\end{tabular} 
& \begin{tabular}[c]{@{}c@{}}Coverage\\ (\%)\end{tabular} 
& \begin{tabular}[c]{@{}c@{}}Avg.\\ width\end{tabular} 
& \begin{tabular}[c]{@{}c@{}}Winkler\\ Score\end{tabular} 
& \begin{tabular}[c]{@{}c@{}}Coverage\\ (\%)\end{tabular} 
& \begin{tabular}[c]{@{}c@{}}Avg.\\ width\end{tabular} 
& \begin{tabular}[c]{@{}c@{}}Winkler\\ Score\end{tabular} 
& \begin{tabular}[c]{@{}c@{}}Coverage\\ (\%)\end{tabular} 
& \begin{tabular}[c]{@{}c@{}}Avg.\\ width\end{tabular} 
& \begin{tabular}[c]{@{}c@{}}Winkler\\ Score\end{tabular} \\ 
\midrule

\multirow{9}{*}{\textbf{Delhi Temp.}}
& ACI
& 91.0 & $\infty$ & $\infty$
& 90.0 & $\infty$ & $\infty$
& 90.2 & $\infty$ & $\infty$
& 90.3 & $\infty$ & $\infty$ \\

& OGD
& 90.4 & 7.54 & 9.78
& 90.1 & 6.82 & \textbf{9.11}
& 90.0 & 6.36 & \textbf{8.48}
& 89.9 & 11.08 & 13.15 \\

& SF-OGD
& 90.0 & 10.63 & 13.24
& 90.1 & 10.92 & 13.72
& 90.1 & 11.33 & 13.99
& 90.0 & 10.93 & 13.64 \\

& decay-OGD
& 90.1 & 11.24 & 13.03
& 90.2 & 10.52 & 12.25
& 90.2 & 10.95 & 12.85
& 89.7 & 14.43 & 16.18 \\

& PID
& 89.9 & 10.87 & 13.36
& 89.7 & 12.28 & 15.11
& 89.7 & 12.49 & 15.13
& 89.9 & 11.57 & 14.15 \\

& ECI
& 90.0 & 7.20 & 9.34
& 90.1 & \underline{6.39} & 9.99
& 90.0 & 6.41 & 8.76
& 89.9 & 9.15 & 12.33 \\

& COP
& 90.1 & 7.05 & 9.35
& 90.1 & \textbf{6.35} & \underline{9.99}
& 90.0 & \textbf{6.27} & 8.69
& 89.9 & 9.12 & 12.36 \\
\cmidrule{2-14}

& \textbf{DDCI}
& 90.0 & \textbf{6.95} & \textbf{9.10}
& 89.9 & 6.44 & 10.04
& 90.0 & 6.36 & 8.69
& 89.9 & \underline{8.57} & \textbf{11.80} \\

& \textbf{DDCI-Nex}
& 90.0 & \underline{7.00} & \underline{9.16}
& 89.9 & 6.39 & 10.10
& 90.0 & \underline{6.31} & \underline{8.67}
& 89.9 & \textbf{8.56} & \underline{11.87} \\

\midrule

\multirow{9}{*}{\textbf{Electricity}}
& ACI
& 90.1 & $\infty$ & $\infty$
& 90.1 & $\infty$ & $\infty$
& 90.2 & $\infty$ & $\infty$
& 90.2 & $\infty$ & $\infty$ \\

& OGD
& 89.8 & 0.433 & 0.504
& 90.0 & 0.133 & 0.188
& 90.1 & 0.081 & 0.117
& 90.1 & 0.139 & 0.181 \\

& SF-OGD
& 89.9 & 0.419 & 0.479
& 90.0 & 0.129 & 0.180
& 90.3 & 0.106 & 0.136
& 90.3 & 0.141 & 0.179 \\

& decay-OGD
& 90.1 & 0.532 & 0.616
& 90.1 & 0.122 & \textbf{0.154}
& 90.0 & 0.100 & 0.132
& 90.3 & 0.147 & \textbf{0.176} \\

& PID
& 90.1 & \textbf{0.373} & \textbf{0.424}
& 90.0 & 0.434 & 0.518
& 89.9 & 0.413 & 0.482
& 89.9 & 0.428 & 0.504 \\

& ECI
& 90.1 & 0.395 & 0.458
& 90.0 & \textbf{0.118} & 0.173
& 89.8 & \textbf{0.069} & \underline{0.111}
& 90.2 & \underline{0.135} & \underline{0.177} \\

& COP
& 90.1 & \underline{0.385} & \underline{0.448}
& 90.0 & \textbf{0.118} & \underline{0.173}
& 89.8 & \textbf{0.069} & 0.112
& 90.2 & 0.136 & 0.178 \\
\cmidrule{2-14}

& \textbf{DDCI}
& 90.1 & 0.392 & 0.454
& 90.0 & \textbf{0.118} & 0.174
& 89.8 & \textbf{0.069} & \textbf{0.111}
& 90.2 & \underline{0.135} & 0.178 \\

& \textbf{DDCI-Nex}
& 90.1 & 0.400 & 0.463
& 90.0 & \textbf{0.118} & 0.174
& 89.8 & \textbf{0.069} & 0.111
& 90.2 & \textbf{0.134} & 0.178 \\

\bottomrule
\end{tabular}}
\end{table*}

\subsubsection{Varying Significance Levels $\alpha$ and Squashing Functions $h(\cdot)$}
\label{results:siglevels}

We conducted experiments on Amazon stock dataset under Prophet model. To comprehensively evaluate the calibration behaviour of the online conformal methods, we conducted a multi-level significance experiment by sweeping $\alpha$ from 0.05 to 0.9. For each level, we measured the empirical coverage, the mean interval width, and the Winkler score across four methods (DDCI, ECI, PID and OGD).

\begin{figure}[h]
    \centering
    \includegraphics[width=0.95\linewidth]{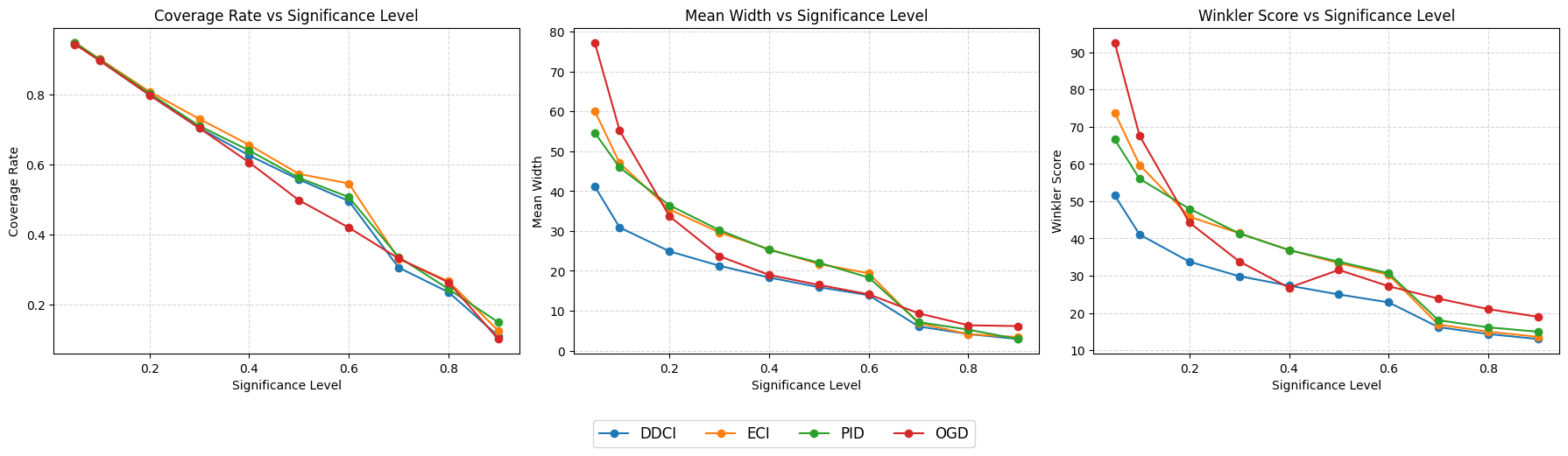}
    \caption{Comparison of empirical coverage rate, mean interval width, and Winkler Score across a wide range of significance levels ($\alpha$ = 0.05–0.9) for four online conformal inference methods: DDCI, ECI, PID, and OGD.}
    \label{fig:placeholder}
\end{figure}

Across the full range of significance levels from 0.05 to 0.9, DDCI consistently demonstrates superior calibration stability and interval efficiency compared with existing online conformal methods. In terms of efficiency, DDCI achieves the narrowest mean interval widths across almost all $\alpha$, reflecting its ability to maintain reliable calibration without resorting to excessively conservative intervals. This advantage is further reinforced by the Winkler Score results: DDCI attains the lowest Winkler Score throughout the entire spectrum, indicating the most favourable balance between coverage accuracy and interval compactness. Overall, the empirical curves of coverage, width, and Winkler Score collectively show that DDCI offers a more robust and efficient update rule, yielding both tighter and better-calibrated prediction intervals under varying levels of uncertainty.

We also tested different squashing functions with different significance level (0.2, 0.1 and 0.05). Table~\ref{table:diff_alpha} shows that changing squashing function does not influence the results significantly. Details are given below.

\begin{table}[h]
\centering
\renewcommand{\arraystretch}{1.2}
\caption{Results under different significance levels $\alpha$ on Amazon stock dataset.}
\label{table:diff_alpha}
\resizebox{\textwidth}{!}{
\begin{tabular}{ccccccccccccc}
\hline
\multirow{2}{*}{\textbf{Method}} 
& \multicolumn{3}{c}{\textbf{$\alpha=0.20$}} 
& \multicolumn{3}{c}{\textbf{$\alpha=0.10$}}
& \multicolumn{3}{c}{\textbf{$\alpha=0.05$}} \\ \cline{2-10}
& \begin{tabular}[c]{@{}c@{}}Coverage\\ (\%)\end{tabular}
& \begin{tabular}[c]{@{}c@{}}Avg.\\ width\end{tabular}
& \begin{tabular}[c]{@{}c@{}}Winkler\\ Score\end{tabular}
& \begin{tabular}[c]{@{}c@{}}Coverage\\ (\%)\end{tabular}
& \begin{tabular}[c]{@{}c@{}}Avg.\\ width\end{tabular}
& \begin{tabular}[c]{@{}c@{}}Winkler\\ Score\end{tabular}
& \begin{tabular}[c]{@{}c@{}}Coverage\\ (\%)\end{tabular}
& \begin{tabular}[c]{@{}c@{}}Avg.\\ width\end{tabular}
& \begin{tabular}[c]{@{}c@{}}Winkler\\ Score\end{tabular} \\
\hline  

OGD 
& 79.7 & 33.70 & 44.27 
& 89.6 & 55.15 & 67.57 
& 94.3 & 77.24 & 92.63 \\

SF-OGD 
& 79.7 & 44.51 & 58.31 
& 89.5 & 61.47 & 73.94 
& 94.5 & 79.11 & 92.74 \\

decay-OGD 
& 80.1 & 46.92 & 57.98 
& 89.9 & 97.22 & 107.79 
& 90.9 & 107.32 & 159.64 \\

PID 
& 80.3 & 36.45 & 47.97 
& 89.8 & 52.56 & 56.02 
& 94.8 & 54.65 & 66.58 \\

ECI 
& 80.8 & 35.39 & 45.85 
& 90.1 & 47.00 & 59.74 
& 94.8 & 60.10 & 73.81 \\

ECI-cutoff 
& 80.0 & 31.81 & 40.93 
& 89.7 & 43.46 & 53.21 
& 94.8 & 55.97 & 66.65 \\

ECI-integral 
& 80.0 & 32.21 & 41.26 
& 89.8 & 42.01 & 51.85 
& 94.8 & 56.00 & 67.19 \\ \hline

\textbf{DDCI} 
& 79.9 & \underline{24.91} & \underline{33.76} 
& 89.7 & \textbf{30.87} & 41.06
& 94.8 & \underline{41.17} & \underline{51.56} \\

\textbf{DDCI-Nex} 
& 79.8 & 25.25 & 34.27 
& 89.7 & \underline{30.92} & \underline{40.95}
& 94.7 & \textbf{41.09} & \textbf{51.32} \\

\textbf{DDCI-Softsign} 
& 80.1 & \textbf{24.56} & \textbf{33.34} 
& 89.7 & 31.11 & \textbf{40.90} 
& 94.8 & 41.23 & 51.77 \\ \hline

\end{tabular}}
\end{table}

\subsubsection{Learning Rate $\eta$ Analysis}
\label{results:learningrates}

In this section, we analyse the impact of different learning rates on the performance of the proposed method, using the Google stock dataset as an example. The learning rate range follows \citet{wu2025error}. For PID, ECI, and DDCI, the base learning rates $\eta$ considered are [1, 0.5, 0.1, 0.05] and these work used the proportional learning rates $\eta_t = {\max_{t-w \leq i,j \leq t} \, \lvert s_i - s_j \rvert} \cdot \eta$. The results for coverage and prediction set of four underlying algorithms are given from Fig.~\ref{fig:cov_ar} to Fig.~\ref{fig:set_transformer}. In our experimental setting, the range $d$ for the squashing function is set to 1. Therefore, the upper bound for $\eta$ should be $\frac{1}{2-\alpha+\epsilon}$. Since $\alpha$ and $\epsilon$ are typically small, the learning rate $\eta$ should not exceed 0.5 (it is obvious that our proposed method loses validity and generates extremely wide prediction intervals when the learning rate is set to 1), which means that the empirical observations are consistent with the theoretical analysis. In general, such learning rate is suitable for all datasets without influencing the validity of intervals. However, such large learning rate always yields conservative prediction intervals for all online conformal methods.

\begin{figure}[H]
    \centering
    \includegraphics[width=\linewidth]{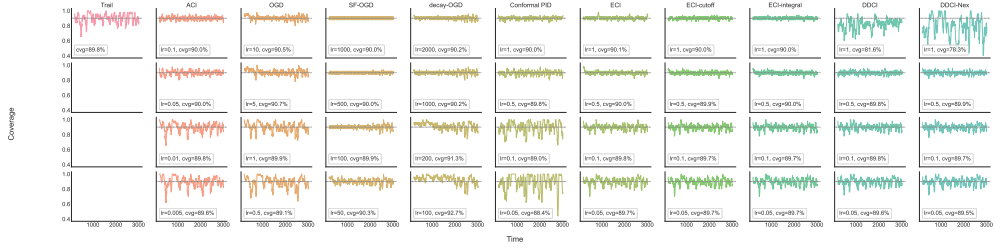}
    \caption{Coverage result for AR model.}
    \label{fig:cov_ar}
\end{figure}

\begin{figure}[H]
    \centering
    \includegraphics[width=\linewidth]{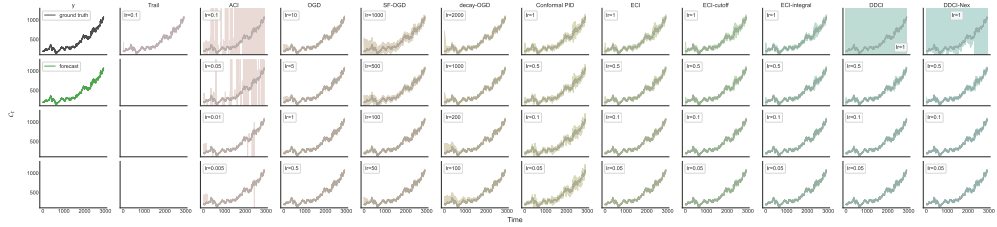}
    \caption{Prediction sets result for AR model.}
    \label{fig:set_ar}
\end{figure}

\vspace{2mm}

\begin{figure}[H]
    \centering
    \includegraphics[width=\linewidth]{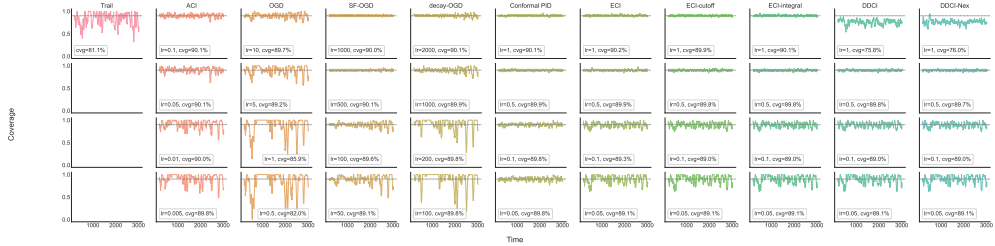}
    \caption{Coverage result for Prophet model.}
    \label{fig:cov_prophet}
\end{figure}

\vspace{2mm}

\begin{figure}[H]
    \centering
    \includegraphics[width=\linewidth]{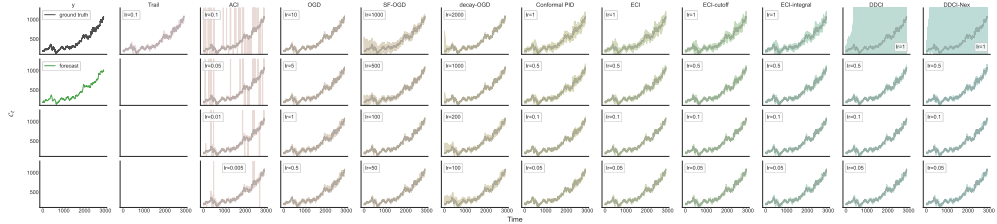}
    \caption{Prediction sets result for Prophet model.}
    \label{fig:set_prophet}
\end{figure}

\vspace{2mm}

\begin{figure}[H]
    \centering
    \includegraphics[width=\linewidth]{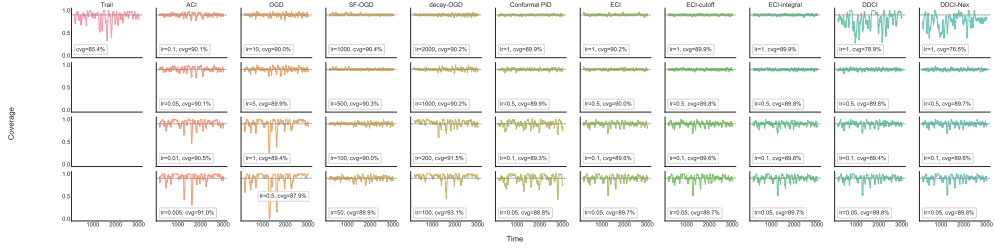}
    \caption{Coverage result for Theta model.}
    \label{fig:cov_theta}
\end{figure}

\vspace{2mm}

\begin{figure}[H]
    \centering
    \includegraphics[width=\linewidth]{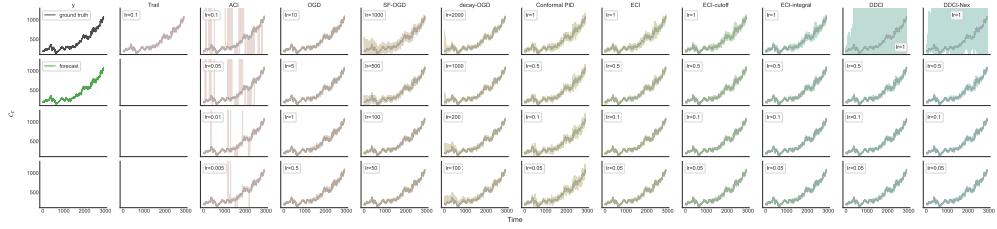}
    \caption{Prediction sets result for Theta model.}
    \label{fig:set_theta}
\end{figure}

\vspace{2mm}

\begin{figure}[H]
    \centering
    \includegraphics[width=\linewidth]{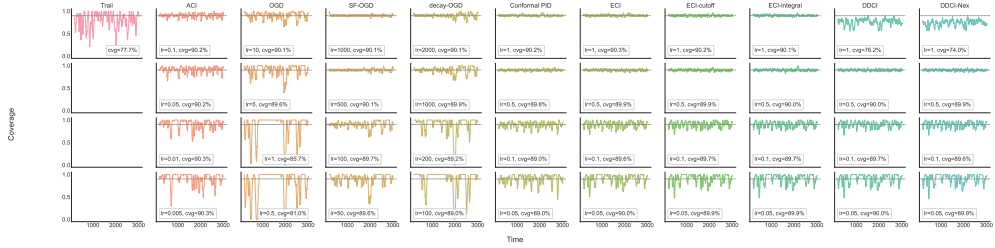}
    \caption{Coverage result for Transformer model.}
    \label{fig:cov_transformer}
\end{figure}

\vspace{2mm}

\begin{figure}[H]
    \centering
    \includegraphics[width=\linewidth]{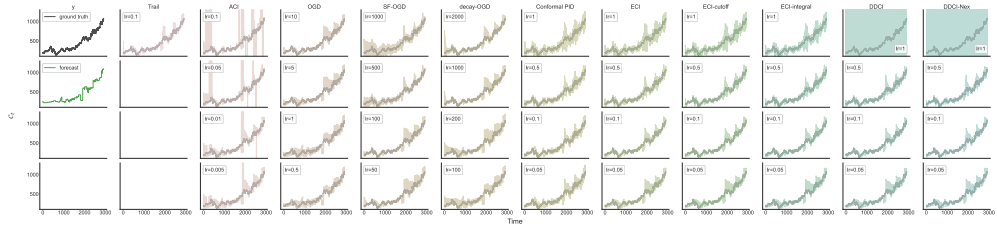}
    \caption{Prediction sets result for Transformer model.}
    \label{fig:set_transformer}
\end{figure}

\subsubsection{Window Length $w$ Analysis}
\label{results:window_length}

We conducted experiments on Google and Amazon datasets, varying the window length $w$ across [50, 100, 200, 300, 500] using Prophet. The results are given in the Fig.~\ref{fig:wl} below. Based on the results, a smaller $w$ (e.g., 50) restricts the "memory" of the estimated feedback. It allows the empirical threshold $q_t^*$ to closely track the current volatility, which causes high miscoverage. A larger $w$ uses a long history to compute $q_t^*$. This introduces unnecessary lag, making the estimated feedback sluggish, thus inflating the interval width.

\begin{figure}[htbp]
    \centering
    \includegraphics[width=\linewidth]{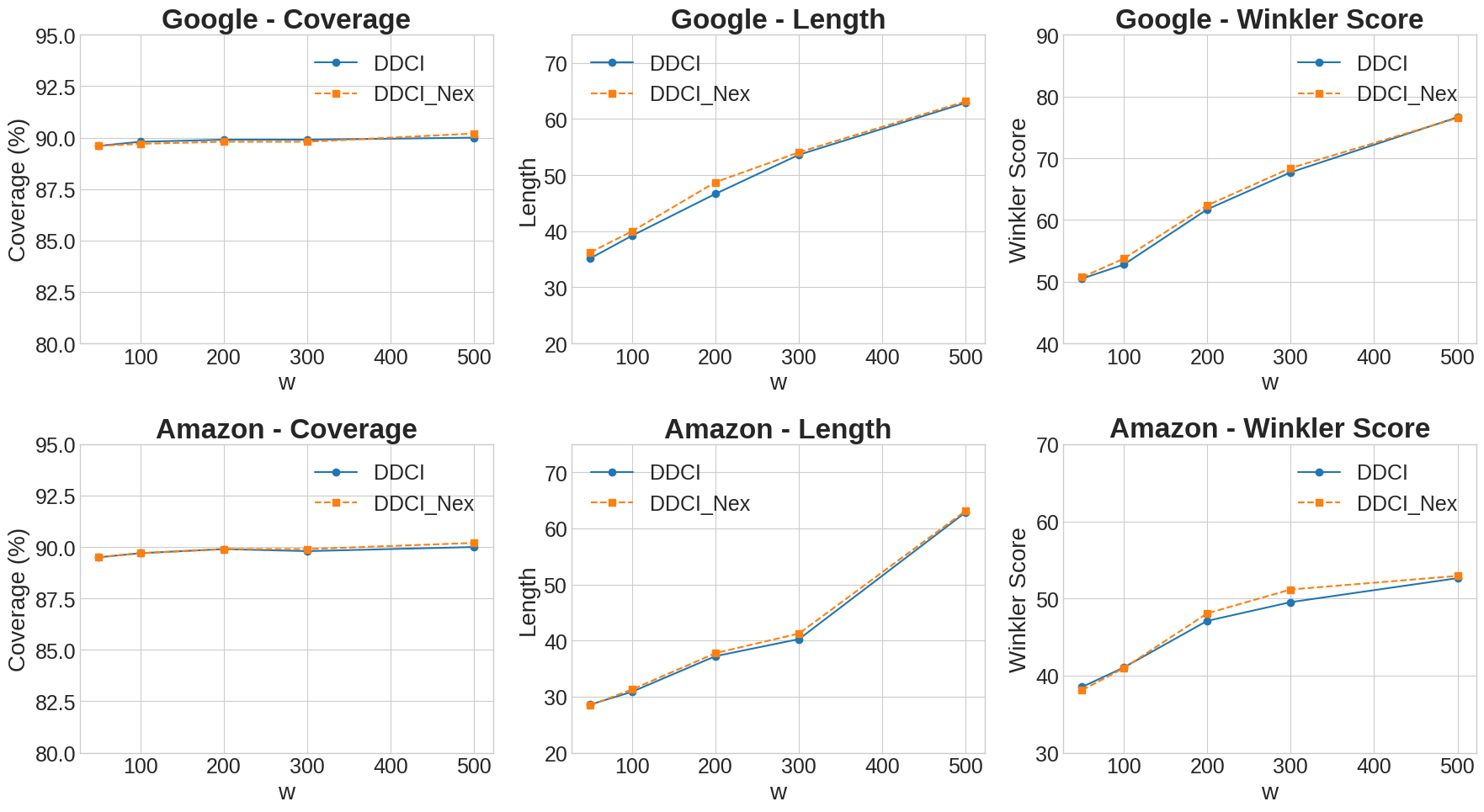}
    \caption{The influence of different window length on Google and Amazon datasets.}
    \label{fig:wl}
\end{figure}

\subsubsection{Squashing Function Range $d$ Analysis}
\label{results:function_range}

We did an analysis on $d$ (ranging from 0.1 to 1.5) on Google dataset under Prophet model. The results show that the value of $d$ would not influence the validity but it is meaningless when it is small. As shown in the Fig.~\ref{fig:ws}, across a wide and practically meaningful range of values, the empirical coverage consistently concentrates tightly around the 90\% target level. However, when $d$ is set to small values, the algorithm effectively degenerates to standard OGD.

\begin{figure}
    \centering
    \includegraphics[width=\linewidth]{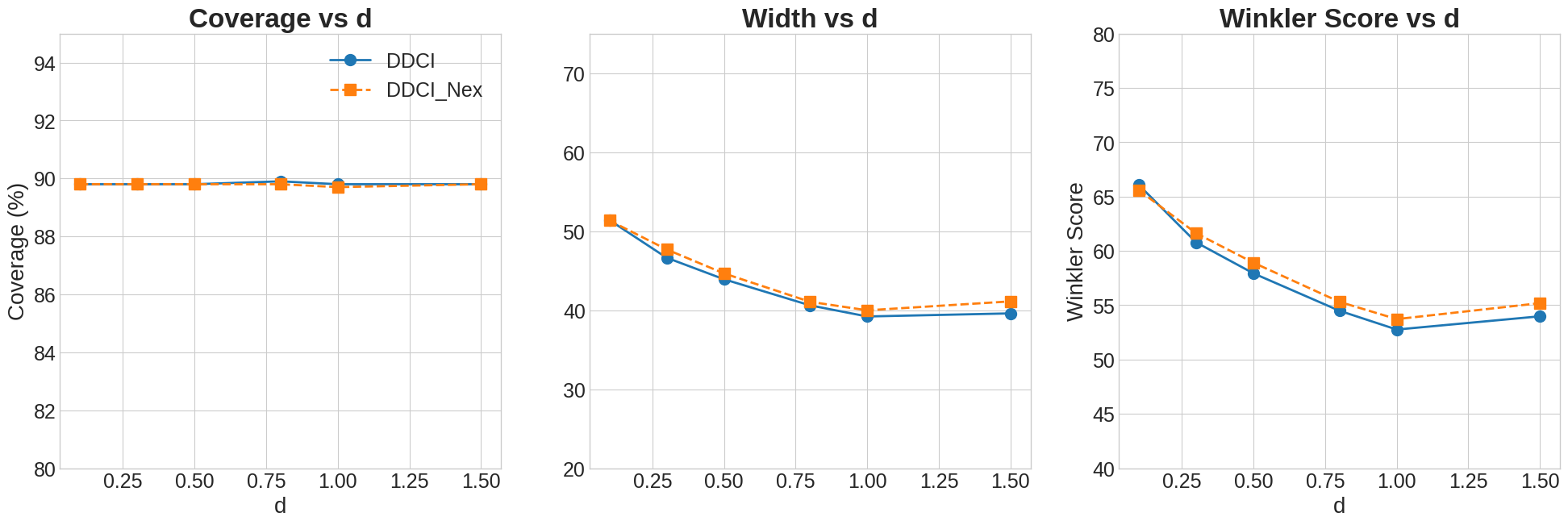}
    \caption{The influence of different $d$ on the Google dataset under the Prophet model.}
    \label{fig:ws}
\end{figure}

\subsubsection{Steepness Parameter $c$ Analysis}
\label{results:steepness}

As illustrated in Fig.~\ref{fig:steepness}, we evaluate the impact of the steepness parameter $c$ (ranging from 0.1 to 1.5) on the Prophet model's forecasting performance using the Google stock dataset. The leftmost subplot clearly demonstrates that the empirical coverage rates for both DDCI and DDCI-Nex remain exceptionally stable, forming a nearly flat line that closely adheres to the 90\% target baseline across the entire range of $c$. Similarly, the middle and right subplots reveal that the Mean Interval Width and Mean Winkler Score exhibit strong robustness to variations in $c$, showing no significant degradation or structural fluctuation. These empirical observations corroborate our worst-case bound in Theorem~\ref{theorem:coverage}, confirming that $c$ controls the theoretical bias without substantially affecting empirical performance.

\begin{figure}
    \centering
    \includegraphics[width=\linewidth]{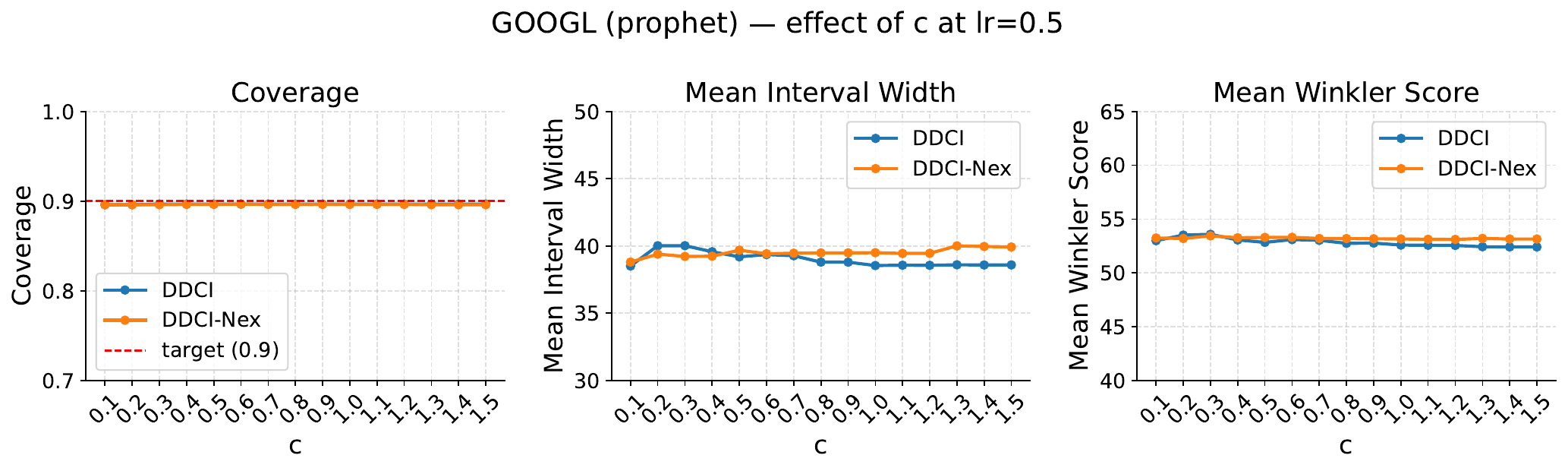}
    \caption{The influence of different $c$ on the Google dataset under the Prophet model.}
    \label{fig:steepness}
\end{figure}

\subsubsection{Ablation Study}
\label{visual:epsilon}

\paragraph{Visualisation.} Fig.~\ref{fig:visualisation_eps} illustrates the effect of different values of $\epsilon$ on the proposed DDCI method through AR model. In this example, we set $\epsilon$ to 0 and 0.8 respectively, which can be regarded as a visualisation for ablation study. When $\epsilon$ is small, the update process becomes more volatile, leading to fluctuations in the prediction intervals. In contrast, larger $\epsilon$ values produce smoother and more stable updates, though at the cost of wider intervals when the learning rate is large. Thus, the choice of $\epsilon$ entails a trade-off between efficiency and stability.

\begin{figure*}[htbp]
    \centering
    \includegraphics[width=0.7\linewidth]{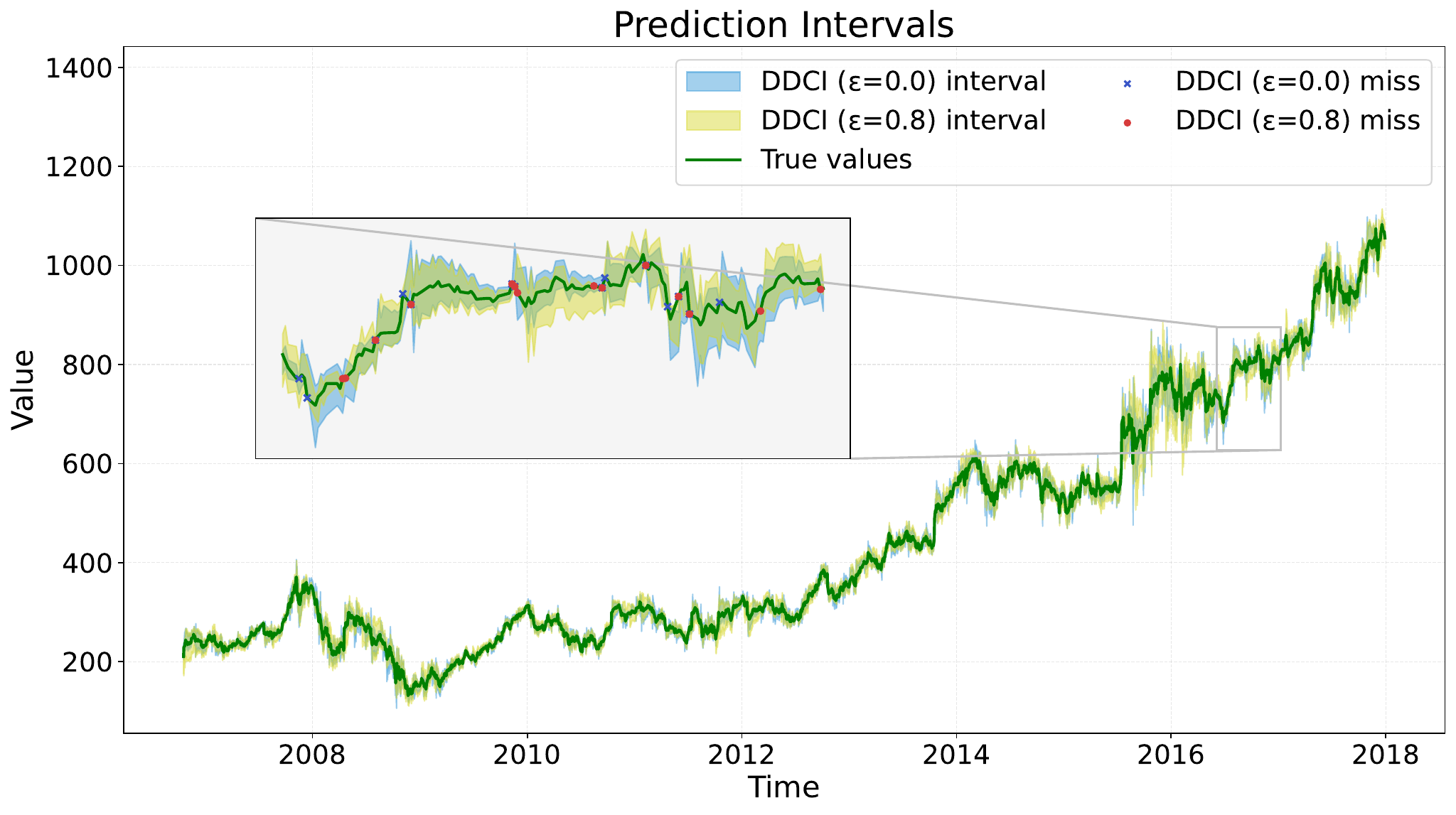}
    \caption{Visualisation of the proposed DDCI method with different $\epsilon$ (0.0 and 0.8 respectively) under AR model.}
    \label{fig:visualisation_eps}
\end{figure*}

\paragraph{Numerical Results.}  We conducted the experiments on all datasets under all models for PID, ECI, and dual/actual/no-feedback methods. The results are given in the Tab.~\ref{tab:additional_stability}. The ablation results clearly validate our \emph{core contribution}: introducing the estimated feedback successfully stabilises the update process. While all three methods maintain the target coverage, our dual-feedback method achieves significantly less oscillation, demonstrating superior stability.

\begin{table*}[t]
\centering
\small
\caption{Additional stability results under AR, Theta, and Transformer. Electricity values are displayed after multiplication by $10^3$. Bold and underline denote the lowest and second-lowest unrounded values, respectively.}
\label{tab:additional_stability}
\resizebox{\textwidth}{!}{%
\begin{tabular}{llccccc}
\toprule
Forecaster & Dataset
& PID
& ECI
& DDCI ($c=0$)
& DDCI ($\epsilon=0$)
& DDCI \\
\midrule

\multirow{5}{*}{AR}
& GOOGL
& 1.3204
& 0.5630
& \underline{0.4707}
& 0.4710
& \textbf{0.4294} \\

& AMZN
& 1.1762
& 0.8711
& 0.7260
& \underline{0.7227}
& \textbf{0.6330} \\

& MSFT
& 0.8166
& 0.5351
& 0.4100
& \underline{0.4034}
& \textbf{0.3768} \\

& Delhi
& 1.0344
& 0.4004
& 0.2860
& \underline{0.2757}
& \textbf{0.2654} \\

& Electricity
& 10.2395
& 4.0433
& 4.0455
& \underline{4.0011}
& \textbf{3.9913} \\
\midrule

\multirow{5}{*}{Theta}
& GOOGL
& 1.6048
& 1.3182
& 1.1391
& \underline{1.1356}
& \textbf{1.0094} \\

& AMZN
& 1.1951
& 0.8819
& 0.7296
& \underline{0.7274}
& \textbf{0.6369} \\

& MSFT
& 0.5412
& 0.1214
& 0.0919
& \textbf{0.0825}
& \underline{0.0842} \\

& Delhi
& 1.0057
& 0.3627
& 0.2546
& \underline{0.2441}
& \textbf{0.2341} \\

& Electricity
& 8.9418
& 2.4094
& 2.4099
& \underline{2.4002}
& \textbf{2.3952} \\
\midrule

\multirow{5}{*}{Transformer}
& GOOGL
& 10.9010
& 11.8597
& \underline{10.4455}
& 12.4885
& \textbf{8.6863} \\

& AMZN
& 7.0344
& 7.8958
& \underline{6.7830}
& 8.3479
& \textbf{5.5897} \\

& MSFT
& 0.9441
& 0.7432
& \underline{0.5823}
& 0.6178
& \textbf{0.5038} \\

& Delhi
& 1.2411
& 0.7240
& 0.5308
& \underline{0.5229}
& \textbf{0.4725} \\

& Electricity
& 11.0325
& 4.9771
& 4.9832
& \underline{4.9430}
& \textbf{4.9241} \\
\bottomrule
\end{tabular}%
}
\end{table*}

Across these 15 additional dataset--forecaster settings, the full DDCI achieves the lowest MASD in 14 settings. The only exception is MSFT under Theta, where the actual-feedback-only configuration obtains a slightly lower MASD ($0.0825$ versus $0.0842$). Nevertheless, the full DDCI outperforms PID, ECI, and the $c=0$ configuration in all 15 settings. Adding estimated feedback reduces MASD relative to actual feedback alone in 14 of the 15 settings. Combined with the five Prophet results in Table~\ref{tab:stability}, DDCI achieves the lowest MASD in 19 of the 20 dataset--forecaster settings.

\section{Reproducibility}
\label{App:code}

All experiments are run with Apple M2 Pro. Here is the anonymous link of the code to reproduce the experimental results and figures: \texttt{https://anonymous.4open.science/r/DDCI-47A7}.

\section{Usage of Language Models}

Large Language Models were used to polish and improve the grammar, spelling, and readability of this manuscript, but all research design, analysis, and interpretation were conducted by the authors.


\newpage
\input{checklist.tex}

\end{document}

%% file: checklist.tex
\section*{NeurIPS Paper Checklist}



\begin{enumerate}

\item {\bf Claims}
    \item[] Question: Do the main claims made in the abstract and introduction accurately reflect the paper's contributions and scope?
    \item[] Answer: \answerYes{} 
    \item[] Justification: The main contributions are summarised at the end of the Section~\ref{sec1:intro}.

\item {\bf Limitations}
    \item[] Question: Does the paper discuss the limitations of the work performed by the authors?
    \item[] Answer: \answerYes{} 
    \item[] Justification: The theoretical limitation of the proposed method is discussed within Section~\ref{mt:TA} and the analysis of the additional parameter $\epsilon$ is given within Section~\ref{appendix:epsilon}.

\item {\bf Theory assumptions and proofs}
    \item[] Question: For each theoretical result, does the paper provide the full set of assumptions and a complete (and correct) proof?
    \item[] Answer: \answerYes{} 
    \item[] Justification: The theoretical analysis is given in the Section~\ref{mt:TA} and the related proof is given in the Appendix~\ref{app:ta}.

    \item {\bf Experimental result reproducibility}
    \item[] Question: Does the paper fully disclose all the information needed to reproduce the main experimental results of the paper to the extent that it affects the main claims and/or conclusions of the paper (regardless of whether the code and data are provided or not)?
    \item[] Answer: \answerYes{} 
    \item[] Justification: The detailed methodology is described in Section~\ref{sec4:methodology} and all experimental settings are provided in Section~\ref{sec:exp_settings}.

\item {\bf Open access to data and code}
    \item[] Question: Does the paper provide open access to the data and code, with sufficient instructions to faithfully reproduce the main experimental results, as described in supplemental material?
    \item[] Answer: \answerYes{} 
    \item[] Justification: The anonymous Github code link is provided in the Appendix~\ref{App:code}.

\item {\bf Experimental setting/details}
    \item[] Question: Does the paper specify all the training and test details (e.g., data splits, hyperparameters, how they were chosen, type of optimizer) necessary to understand the results?
    \item[] Answer: \answerYes{} 
    \item[] Justification: All experimental settings are provided in the Section~\ref{sec:exp_settings}

\item {\bf Experiment statistical significance}
    \item[] Question: Does the paper report error bars suitably and correctly defined or other appropriate information about the statistical significance of the experiments?
    \item[] Answer: \answerYes{} 
    \item[] Justification: Most online CP methods (including ACI\citep{gibbs2021adaptive}, PID\citep{angelopoulos2023pid}, ECI\citep{wu2025error}) do not report variance because the update rules are deterministic and our experiments follows their settings strictly. For our synthetic datasets, we used 5 random seeds to generate three settings and reported the variance.

\item {\bf Experiments compute resources}
    \item[] Question: For each experiment, does the paper provide sufficient information on the computer resources (type of compute workers, memory, time of execution) needed to reproduce the experiments?
    \item[] Answer: \answerYes{} 
    \item[] Justification: All experiments are conducted with Apple M2 Pro.
    
\item {\bf Code of ethics}
    \item[] Question: Does the research conducted in the paper conform, in every respect, with the NeurIPS Code of Ethics \url{https://neurips.cc/public/EthicsGuidelines}?
    \item[] Answer: \answerYes{} 
    \item[] Justification: The research conducted in the paper conform with the NeurIPS Code of Ethics.

\item {\bf Broader impacts}
    \item[] Question: Does the paper discuss both potential positive societal impacts and negative societal impacts of the work performed?
    \item[] Answer: \answerNA{} 
    \item[] Justification: There is no societal impact of the work performed.
    
\item {\bf Safeguards}
    \item[] Question: Does the paper describe safeguards that have been put in place for responsible release of data or models that have a high risk for misuse (e.g., pre-trained language models, image generators, or scraped datasets)?
    \item[] Answer: \answerNA{} 
    \item[] Justification: The paper poses no such risks.

\item {\bf Licenses for existing assets}
    \item[] Question: Are the creators or original owners of assets (e.g., code, data, models), used in the paper, properly credited and are the license and terms of use explicitly mentioned and properly respected?
    \item[] Answer: \answerYes{} 
    \item[] Justification: The baselines and datasets used in this paper are drawn from the open-source implementations and the original works \citep{angelopoulos2023pid, wu2025error} are cited.

\item {\bf New assets}
    \item[] Question: Are new assets introduced in the paper well documented and is the documentation provided alongside the assets?
    \item[] Answer: \answerYes{} 
    \item[] Justification: The model introduced in the paper is well documented and the documentation is provided in the anonymous Github link in Appendix~\ref{App:code}.

\item {\bf Crowdsourcing and research with human subjects}
    \item[] Question: For crowdsourcing experiments and research with human subjects, does the paper include the full text of instructions given to participants and screenshots, if applicable, as well as details about compensation (if any)? 
    \item[] Answer: \answerNA{} 
    \item[] Justification: The paper does not involve crowdsourcing nor research with human subjects.

\item {\bf Institutional review board (IRB) approvals or equivalent for research with human subjects}
    \item[] Question: Does the paper describe potential risks incurred by study participants, whether such risks were disclosed to the subjects, and whether Institutional Review Board (IRB) approvals (or an equivalent approval/review based on the requirements of your country or institution) were obtained?
    \item[] Answer: \answerNA{} 
    \item[] Justification: The paper does not involve crowdsourcing nor research with human subjects.

\item {\bf Declaration of LLM usage}
    \item[] Question: Does the paper describe the usage of LLMs if it is an important, original, or non-standard component of the core methods in this research? Note that if the LLM is used only for writing, editing, or formatting purposes and does \emph{not} impact the core methodology, scientific rigor, or originality of the research, declaration is not required.
    \item[] Answer: \answerNA{} 
    \item[] Justification: LLM is used only for writing and editing.

\end{enumerate}